\documentclass{article}
\usepackage[T1]{fontenc}
\usepackage{iclr2027_conference,times}

\usepackage{graphicx}
\usepackage{wrapfig}
\usepackage{pifont}

\usepackage{amsmath,amsfonts,bm}

\def\1{\bm{1}}

\DeclareMathAlphabet{\mathsfit}{\encodingdefault}{\sfdefault}{m}{sl}
\SetMathAlphabet{\mathsfit}{bold}{\encodingdefault}{\sfdefault}{bx}{n}

\usepackage{graphicx,booktabs,amssymb,xcolor}
\newtheorem{theorem}{Theorem}
\newcounter{hirpalgorithm}
\newenvironment{hirpalgorithm}[1]{%
  \par\medskip\refstepcounter{hirpalgorithm}\noindent\begin{minipage}{\linewidth}
  \hrule\smallskip\textbf{Algorithm \thehirpalgorithm: #1}\par\smallskip\hrule\smallskip\small
}{\par\smallskip\hrule\end{minipage}\par\medskip}
\usepackage{hyperref}
\usepackage{url}

\allowdisplaybreaks[4]

\title{Hierarchical Response Preservation for Continual Adaptation of Zero-Shot Graph--Text Models}

\newcommand{\hirpauthor}[3][School of Computer Science and\\Engineering, Beihang University]{%
  \begin{minipage}[t]{0.32\textwidth}
    \centering
    {\large\normalfont #2\par}
    \vspace{2pt}
    {\small\normalfont
    #1\\
    Beijing, China\\
    \urlstyle{same}\href{mailto:#3}{\nolinkurl{#3}}\par}
  \end{minipage}%
}

\newcommand{\hirpauthors}{%
  \hirpauthor{Haopeng Zhang}{24373469@buaa.edu.cn}\hfill
  \hirpauthor{Yuhan Wang}{yuhanwang@buaa.edu.cn}\hfill
  \hirpauthor[School of Electronic and\\Information Engineering\\Beihang University]{Yubing Su}{24373213@buaa.edu.cn}\par
  \vspace{10pt}
  \hirpauthor{Yingxin Chen}{chenyx1109@buaa.edu.cn}\hfill
  \hirpauthor[School of Software, Beihang University]{Xiao Wang}{xiao_wang@buaa.edu.cn}\hfill
  \hirpauthor{Ruijie Wang\thanks{Corresponding author.}}{ruijiew@buaa.edu.cn}\par
  \vspace{10pt}
  \hirpauthor{Jianxin Li}{lijx@buaa.edu.cn}%
}
\author{\hirpauthors}

\makeatletter
\renewcommand{\@maketitle}{\vbox{\hsize\textwidth
  {\LARGE\scshape\@title\par}
  \vskip 14pt
  {\centering\@author\par}
  \vskip 14pt
}}
\makeatother

\iclrfinalcopy
\hypersetup{
  hidelinks,
  pdftitle={Hierarchical Response Preservation for Continual Adaptation of Zero-Shot Graph--Text Models},
  pdfauthor={Haopeng Zhang, Yuhan Wang, Yubing Su, Yingxin Chen, Xiao Wang, Ruijie Wang, Jianxin Li},
  pdfsubject={Preprint},
  pdfkeywords={continual learning, graph-text models, zero-shot learning, HiRP}
}

\begin{document}
\maketitle
\lhead{Preprint}

\begin{abstract}
Pretrained graph--text models align graph representations with textual semantics, enabling recognition of unseen classes and transfer across graph domains.
However, as graph data and classes continually arrive, models should learn from new supervision while retaining their zero-shot transfer capabilities and historical task knowledge.
Two challenges arise: (i) new classes can overturn historical predictions despite preserved distinctions among historical classes, and (ii) overly strict response preservation can stall learning of new classes.
To address these challenges, we propose \textbf{Hi}erarchical \textbf{R}esponse \textbf{P}reservation (HiRP).
HiRP represents this competition through a hierarchical response that keeps each historical-class probability and sums new-class probabilities, preserving historical distinctions and aggregate competition while allowing distinctions within the new class group to adapt.
It further uses the geometry induced by this response to guide constrained updates, retaining useful adaptation directions while controlling response drift.
Across three class-incremental settings, HiRP achieves
absolute gains of 1.84--7.95 percentage points in average accuracy over the
strongest compared baseline in each setting, while mitigating
zero-shot transfer degradation.\footnote{We will release code and data upon acceptance.}
\end{abstract}

\section{Introduction}
\label{sec:introduction}

Pretrained graph--text models enable zero-shot node classification by matching node representations to textual class descriptions in a shared semantic space~\citep{zhu2025graphclip,liu2026adaligner}. In evolving environments, however, labeled classes may arrive sequentially from different graph domains, requiring these models to acquire new knowledge beyond their pretrained capabilities. We study \emph{cross-graph class-incremental adaptation}: each task introduces a disjoint set of classes, and inference must distinguish all classes encountered so far without access to task identities. The goal is to learn from incoming supervision while retaining both historical-class performance and zero-shot transfer to graph domains excluded from adaptation.

\begin{wrapfigure}[14]{R}{0.43\textwidth}
    \centering
    \includegraphics[width=\linewidth]{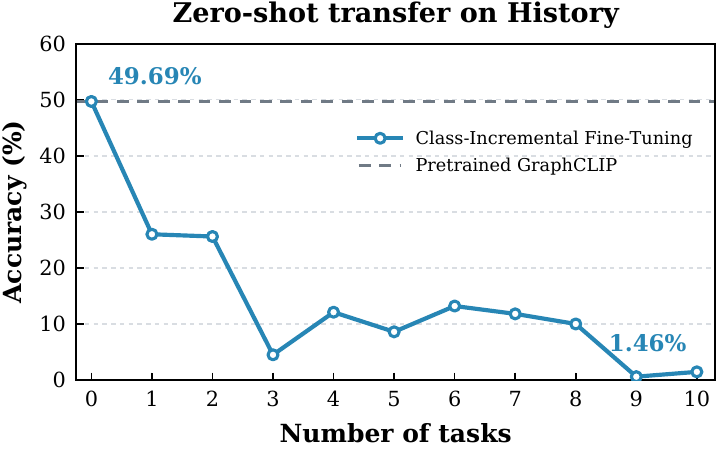}
    \caption{Zero-shot accuracy of GraphCLIP \citep{zhu2025graphclip} on History.}
    \label{fig:zero_shot_degradation}
\end{wrapfigure}

The central challenge is to reconcile new knowledge acquisition with knowledge preservation in a semantic prediction space that continually expands. Supervised updates can disrupt both the graph--text alignment established during pretraining and the distinctions learned from earlier tasks. The vulnerability of pretrained transfer is substantial: in our ten-task setting, joint fine-tuning reduces GraphCLIP's zero-shot accuracy on the held-out History domain from $49.69\%$ to $1.46\%$ (Figure~\ref{fig:zero_shot_degradation}). Meanwhile, newly introduced class descriptions change the competition faced by historical classes. Preserving previously learned distinctions is therefore not enough: those distinctions must continue to support correct predictions against an expanding set of competitors.

Existing continual learning methods mitigate forgetting through parameter regularization~\citep{kirkpatrick2017ewc}, functional preservation~\citep{li2016learning,titsias2020functional}, and gradient projection~\citep{farajtabar2020ogd,wang2026g2lora}. Related work addresses historical--new class separation~\citep{hou2019lucir}, probability aggregation under changing label semantics~\citep{cermelli2020mib}, and functional trust regions~\citep{schulman2015trpo}. However, simply restricting model changes or preserving historical outputs does not ensure prediction stability when new competitors enter the candidate space. At the other extreme, preserving the complete expanded response can restrict distinctions that should be learned among new classes. Moreover, a bound on functional drift does not by itself determine how to continue learning near its boundary. These limitations call for a selective preservation target for expanding class competition and an update rule that exploits the adaptation freedom left by that target, giving rise to two interconnected challenges.

\noindent
\begin{minipage}[t]{0.43\textwidth}
\vspace{0pt}
\textbf{Challenge 1: Preserving historical predictions under expanding class competition.}
A historical node's correct class may still outrank every other historical class yet be overtaken by a newly introduced class (Figure~\ref{fig:challenges}). Thus, stable relative preferences among historical classes need not yield stable predictions in the expanded space. The challenge is to identify a response that protects both historical distinctions and their competition with new classes, without unnecessarily constraining how the model distinguishes classes within the new group.
\end{minipage}\hfill
\begin{minipage}[t]{0.55\textwidth}
    \vspace{0pt}
    \centering
    \makeatletter\def\@captype{figure}\makeatother
    \includegraphics[width=\linewidth]{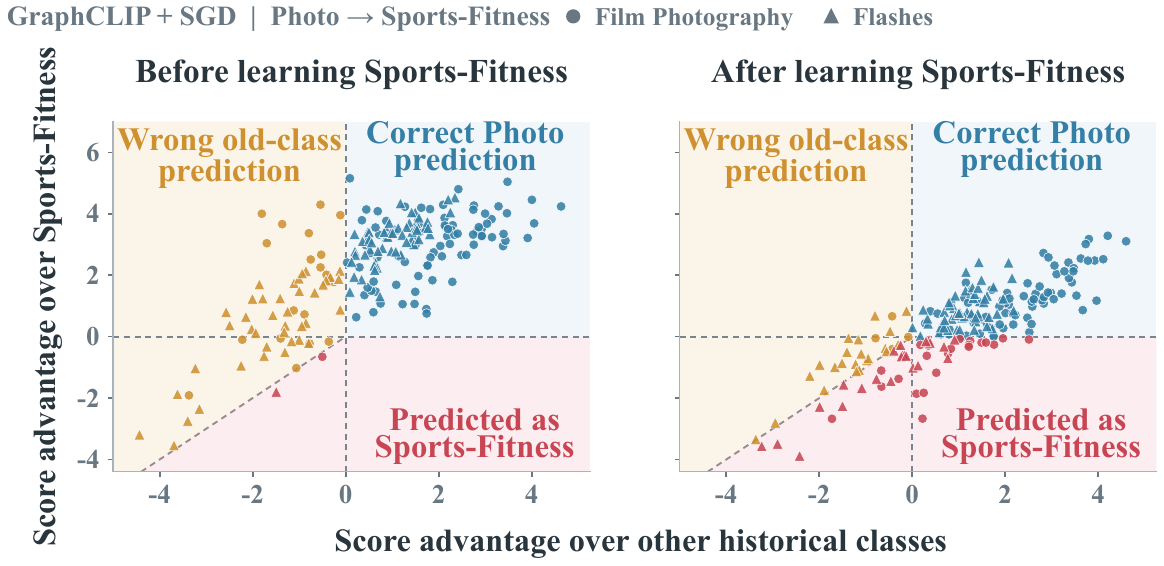}
    \caption{Evolving class competition: new classes can
    overturn historical predictions despite preserved
    discrimination among historical classes.}
    \label{fig:challenges}
\end{minipage}
\par\medskip

\textbf{Challenge 2: Learning new classes without stalling model updates.}
Once a response is protected, adaptation must operate within its allowed drift. This restriction is inherently directional: some update components strongly alter the protected probabilities, whereas others can improve new-task performance with less disruption. Near the drift limit, uniformly shrinking an update suppresses both kinds of components and can stall learning. The challenge is to retain useful adaptation directions while keeping changes to the protected response within the prescribed limits.

To address these challenges, we propose \emph{Hierarchical Response Preservation} (HiRP), a continual adaptation framework for pretrained graph--text models. The key idea is to use \emph{the same hierarchical response to define both what is preserved and how adaptation proceeds}. HiRP constructs this response on unlabeled support features and derives its preservation constraints and local update geometry from that response, connecting the relationships to be protected with the directions available for learning.

\textbf{To address Challenge 1, HiRP preserves historical class-level responses and aggregate new-class competition.}
For each support node, it retains an individual probability for every historical class and aggregates the probabilities of all new classes into a single component. Constraining this hierarchical response protects historical distinctions and limits changes in the total probability assigned to new competitors. Crucially, it leaves probability redistribution within the new-class group unconstrained by the preservation target, allowing new supervision to refine those distinctions. We further establish a sufficient condition connecting response preservation to prediction stability.

\textbf{To address Challenge 2, HiRP guides adaptation using the geometry of the protected response.}
The same response Kullback--Leibler (KL) divergence defines per-step and task-level drift constraints, as well as a local update metric that favors directions with less impact on protected relationships. Near the task-level drift limit, tangent correction suppresses the divergence-increasing component of an update while retaining motion along the constraint boundary. Nonlinear KL feasibility checks verify both drift limits before an update is accepted. This design seeks useful new-task progress within the preservation constraints rather than relying solely on uniform step shrinkage.

Our key contributions are summarized as follows.
\ding{182}~\underline{\textbf{\textit{Framework.}}}
We propose HiRP, a continual learning framework for pretrained graph--text models that jointly addresses new knowledge acquisition, zero-shot transfer retention, and historical knowledge preservation as classes arrive across graph domains.
\ding{183}~\underline{\textbf{\textit{Mechanism.}}}
We design hierarchical response preservation and response-geometry-guided adaptation, representing protected competitive relationships through historical class-level and new group-level responses.
The same response KL defines drift constraints and an update metric, while tangent correction and nonlinear feasibility checks enable constrained updates.
We further establish a sufficient condition for prediction stability.
\ding{184}~\underline{\textbf{\textit{Experiments.}}}
We evaluate HiRP across three class-incremental settings against 20 baselines. HiRP improves average accuracy by \textbf{1.84--7.95 percentage points} over the strongest compared baseline in each setting, while mitigating zero-shot transfer degradation during continual adaptation.

\section{Related Work}
\label{sec:related}

\paragraph{Continual learning and graph adaptation.}
Replay methods retain exemplars or historical inputs and logits \citep{rebuffi2017icarl,buzzega2020dark}; EWC and MAS penalize changes to important parameters \citep{kirkpatrick2017ewc,aljundi2018mas}, while OGD locally protects outputs through gradient projection \citep{farajtabar2020ogd}.
L2P learns task-agnostic prompts \citep{wang2022l2p}; graph TPP predicts tasks through profiling and prompting \citep{niu2024tpp}.
LUCIR addresses historical/current-class imbalance and separation \citep{hou2019lucir}, and G2LoRA protects category-level subspaces while coordinating graph/text adapter updates \citep{wang2026g2lora}.
HiRP constrains a shared readout's response on unlabeled features; inference directly scores all arrived classes.

\paragraph{Distillation and functional memory.}
LwF distills historical responses on current inputs \citep{li2016learning}; FRCL and FROMP preserve functions at inducing inputs or memorable examples \citep{titsias2020functional,pan2020fromp}.
MiB aggregates current-class and background probabilities to match a previous segmentation model's background label \citep{cermelli2020mib}.
DKD separates target-versus-rest mass and non-target conditional probabilities \citep{zhao2022dkd}.
HiRP uses pretrained class semantics to evaluate both references on the same expanded candidates: its aggregate selects current-class distinctions to leave unconstrained, rather than reconciling changing background labels as in MiB.
The contribution couples this response choice with local and task-level drift control and feasible adaptation on unlabeled features.
Aggregation and the KL chain rule are established tools; Appendix~\ref{app:response-comparison} compares protected quantities and support.

\paragraph{Trust-region updates.}
TRPO limits successive policy KL through a local quadratic approximation \citep{schulman2015trpo}; Natural Continual Learning shapes updates by prior precision \citep{kao2021ncl}.
TRGP selects related historical tasks for gradient projection and scaled weight reuse \citep{lin2022trgp}, while a recent preprint couples generative replay with a Fisher trust region \citep{wang2026trcl}.
HiRP uses the selected response for its metric and nonlinear KL checks against the current iterate and task-start reference, redirecting outward proposals near the task budget.

\section{Problem Setup}
\label{sec:setup}

\paragraph{Learning protocol.}
We consider a sequence of $T$ class-incremental tasks.
Task $t$ introduces new classes $\mathcal C_t$, disjoint from earlier class sets, with labeled training data $\mathcal D_t$.
Each input $x$ denotes a node together with its graph context. A task may introduce classes from a new graph domain or a new class subset of an existing domain.
The graph instances are fixed and transductive: graph structure and unlabeled attributes, including neighbors from other tasks or splits, may be used to compute features.
Only current-task training labels enter the supervised loss; validation/test nodes never enter the feature support.
We call $\mathcal O_t=\bigcup_{i<t}\mathcal C_i$ the \emph{historical classes}, $\mathcal N_t=\mathcal C_t$ the \emph{current classes}, and $\mathcal C_{\leq t}=\mathcal O_t\cup\mathcal N_t$ the cumulative candidates.
Task boundaries are known, and each task uses its complete prescribed training split.
Inference considers all classes in $\mathcal C_{\leq t}$ without a ground-truth task identity.

\paragraph{Competition in a shared semantic space.}
HiRP freezes the pretrained GraphCLIP graph and text encoders.
Let $\mathbf{h}(x)\in\mathbb R^d$ be the frozen node feature, $\widetilde{\mathbf{h}}(x)=[\mathbf{h}(x);1]$ its augmentation, and $\mathbf{u}_c\in\mathbb R^m$ the unit-normalized text embedding of class $c$.
Only the shared readout $\mathbf{W}\in\mathbb R^{m\times(d+1)}$ is updated.
With a fixed scale $\gamma>0$, the class score is
\begin{equation}
    s_c(x;\mathbf{W})=\gamma \mathbf{u}_c^\top
    \frac{\mathbf{W}\widetilde{\mathbf{h}}(x)}{\lVert \mathbf{W}\widetilde{\mathbf{h}}(x)\rVert_2}.
    \label{eq:class-score}
\end{equation}
The predictive distribution and decision are
\begin{equation}
    p_t(c\mid x;\mathbf{W})=
    \frac{\exp s_c(x;\mathbf{W})}{\sum_{j\in\mathcal C_{\leq t}}\exp s_j(x;\mathbf{W})},
    \qquad
    \hat y_t(x)=\arg\max_{c\in\mathcal C_{\leq t}}s_c(x;\mathbf{W}).
    \label{eq:global-prediction}
\end{equation}
Candidate arrival expands the label set, while the embedding dimension stays fixed.
The frozen text encoder supplies $\mathbf{u}_c$ as soon as a class is introduced, so its competition on historical features is observable before adaptation.
Training uses cross-entropy on $\mathcal D_t$ over these same candidates.

\paragraph{Bounded feature support.}
During task $t$, a support $\mathcal M_t$ contains at most $M$ frozen feature vectors drawn from an external unlabeled reference pool and previous training tasks.
Source and task metadata are retained, but historical node labels are not used in training losses.
Support features serve only to evaluate functional constraints and their local metric.
After each task, features from $\mathcal D_t$ are incorporated within the fixed support budget.
The support construction and update constraints are specified in the method section.

\section{Motivating Observations}
\label{sec:motivation}

\paragraph{Historical discrimination and cross-group competition.}
For each task transition, we evaluate the same historical test samples before and after training.
Write $A_{-,O}$ and $A_{+,O}$ for their accuracies among historical candidates $\mathcal O_t$, and $A_{-,S}$ and $A_{+,S}$ for accuracies among cumulative candidates $\mathcal C_{\leq t}$.
The accuracy change admits the identity
\begin{equation}
    A_{-,O}-A_{+,S}
    =\underbrace{A_{-,O}-A_{+,O}}_{\text{historical discrimination change}}
    +\underbrace{A_{+,O}-A_{+,S}}_{\text{post-update competition gap}}.
    \label{eq:accuracy-decomposition}
\end{equation}
\noindent
\begin{minipage}[t]{0.66\linewidth}
    \vspace{0pt}
We average over task transitions within each run before aggregating runs.
For SGD, the competition gap ranges from $17.58$ to $22.45$ percentage points across the three streams; historical discrimination changes by $-1.81$ to $1.66$ points.
Stable accuracy among historical candidates can therefore coexist with substantial errors under unified prediction.

Candidate expansion also causes errors before training.
In the ten-task setting, adding current candidates at fixed pre-update parameters reduces accuracy from $A_{-,O}=63.84\%$ to $A_{-,S}=60.82\%$ (Figure~\ref{fig:motivation-competition}).
This $3.02$-point loss is distinct from the $20.61$-point post-update competition gap.
The pre-update loss measures candidate insertion; the post-update gap includes both insertion and subsequent adaptation.
\end{minipage}\hfill
\begin{minipage}[t]{0.30\linewidth}
    \vspace{0pt}
    \centering
    \makeatletter\def\@captype{figure}\makeatother
    \includegraphics[width=\linewidth]{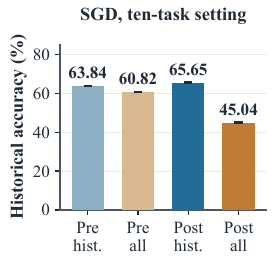}
    \caption{Historical accuracy with expanded candidates.}
    \label{fig:motivation-competition}
\end{minipage}
\par\vspace{-11pt}

\begin{wrapfigure}{R}{0.54\textwidth}
    \vspace{-16pt}
    \centering
    \includegraphics[width=\linewidth]{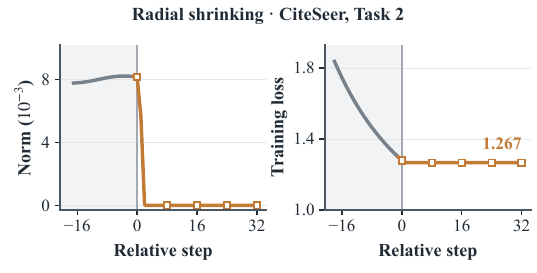}
    \caption{Uniform step shrinking near the KL limit reduces update norms while loss stalls.}
    \label{fig:motivation-radial-stalling}
    \vspace{-6pt}
\end{wrapfigure}
\paragraph{Uniform step shrinking can stall new-task learning.}
We uniformly scale down each proposed update without changing its direction (radial shrinking) on a 16-node training minibatch from CiteSeer, Task 2 of the four-task setting.
Figure~\ref{fig:motivation-radial-stalling} tracks the accepted update norm and training loss as the trajectory approaches the task-level KL boundary and the local step-radius budget decreases.
The update norm collapses toward zero, while the loss plateaus at $1.267$.
Satisfying the protection budget can therefore suppress adaptation even when the current gradient is nonzero.
Appendix~\ref{app:boundary-trajectories} gives the diagnostic setup; Section~\ref{sec:case-study} examines how tangent correction affects this behavior.

These observations motivate a response that protects historical discrimination and aggregate cross-group competition while retaining useful adaptation directions.
Protection acts on drift from the expanded-candidate reference, which can already contain candidate-insertion errors.

\section{HiRP: Hierarchical Response Preservation}
\label{sec:method}

In this section, we introduce HiRP, a continual adaptation method for pretrained graph--text models that learns new classes while preserving historical performance and zero-shot transfer.
As illustrated in Figure~\ref{fig:method-overview}, HiRP comprises three connected steps:
(1) a shared representation that freezes the graph and text encoders and adapts only the readout $\mathbf{W}$;
(2) a hierarchical response that keeps each historical-class probability and sums current-class probabilities to capture competition between historical and new classes; and
(3) a KL-constrained update that limits changes in this response and guides the readout toward useful learning directions.
We write $\boldsymbol{\theta}=\operatorname{vec}(\mathbf{W})$, with $\boldsymbol{\theta}_0$ denoting the task-entry parameters and $\boldsymbol{\theta}_k$ the current iterate.
The full procedure is given in Algorithm~\ref{alg:hirp}.

\begin{figure}[t]
\centering
\includegraphics[width=\linewidth]{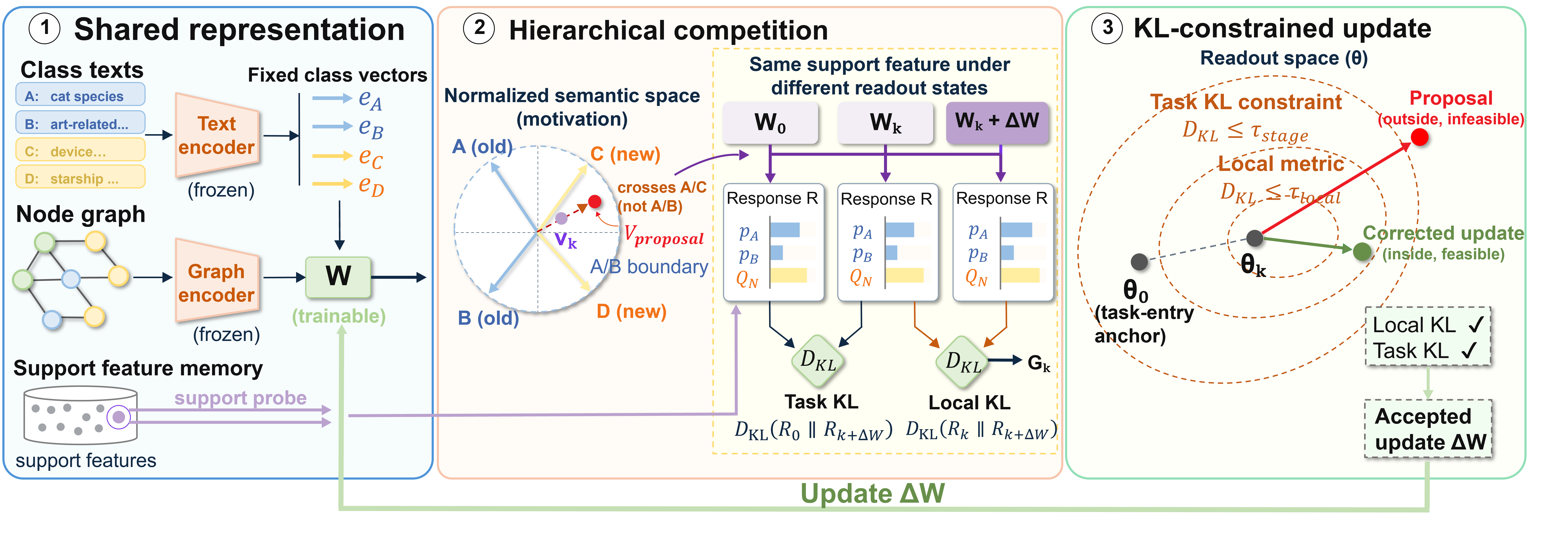}
\caption{\textbf{Overview of HiRP.}
\textbf{(1) Shared representation:} Frozen encoders provide node and class features; only the readout is updated.
\textbf{(2) Hierarchical competition:} Historical classes remain separate, while new classes share an aggregate response.
\textbf{(3) Constrained adaptation:} The response defines drift budgets and guides feasible readout updates.}
\label{fig:method-overview}
\end{figure}

\subsection{Selecting the response to preserve}
\label{sec:hierarchical-response}

\paragraph{Encoding competitive relationships.}
Predictive probabilities encode class competition: for $\boldsymbol{\nu}=\mathbf{W}\widetilde{\mathbf{h}}/\|\mathbf{W}\widetilde{\mathbf{h}}\|$, scores satisfy
\begin{equation}
    \log\frac{p_i}{p_j}=s_i-s_j=\gamma(\mathbf{u}_i-\mathbf{u}_j)^\top\boldsymbol{\nu}.
    \label{eq:geometry-log-odds}
\end{equation}
Probability ratios therefore encode pairwise score margins exactly, without requiring probability calibration (Appendix~\ref{app:competition-geometry}).

Suppressing the task index, let $p_c=p_t(c\mid x;\boldsymbol{\theta})$, $q_O=\sum_{o\in\mathcal O_t}p_o$, and $q_N=\sum_{n\in\mathcal N_t}p_n$.
HiRP keeps each historical-class probability and represents current arrivals by their total probability:
\begin{equation}
    \mathbf{r}_t(x;\boldsymbol{\theta})=\big((p_o)_{o\in\mathcal O_t},q_N\big),
    \qquad
    D_H(x;\boldsymbol{\theta}^a,\boldsymbol{\theta})
    =D_{\mathrm{KL}}\!\left(\mathbf{r}_t(x;\boldsymbol{\theta}^a)\,\|\,\mathbf{r}_t(x;\boldsymbol{\theta})\right).
    \label{eq:hierarchical-response}
\end{equation}
For $t>1$, define the historical conditional distribution $r_O(o)=p_o/q_O$.
A superscript $a$ denotes quantities evaluated at reference parameters $\boldsymbol{\theta}^a$.
The KL chain rule makes the protected relationships explicit:
\begin{equation}
    D_H
    =\underbrace{D_{\mathrm{KL}}\!\left((q_O^a,q_N^a)\,\|\,(q_O,q_N)\right)}_{\text{aggregate group competition}}
    +\underbrace{q_O^a D_{\mathrm{KL}}(\mathbf{r}_O^a\|\mathbf{r}_O)}_{\text{historical conditional structure}}.
    \label{eq:hierarchical-chain-rule}
\end{equation}
\paragraph{Which probabilities are constrained, and which can change.}
For historical A/B and current C/D, $\mathbf{r}=(p_A,p_B,p_C+p_D)$.
Constraining this response limits changes in historical probabilities and the total probability assigned to current classes, while C/D can redistribute that total.
This is the coarse competition structure HiRP retains from Figure~\ref{fig:challenges}: current classes compete as a group, without fixing each new-class probability.
For example, $(0.40,0.10,0.25,0.25)\to(0.40,0.10,0.49,0.01)$ keeps $\mathbf{r}$ unchanged but switches the winner from A to C.
Prediction stability follows when the historical winner has sufficient margin over this aggregate.

\begin{theorem}[Historical prediction stability]
\label{prop:response-stability}
For $t>1$ and a reference historical winner $c$, let
\begin{equation}
    \delta_R=p_c^a-\max\!\left\{\max_{o\in\mathcal O_t\setminus\{c\}}p_o^a,q_N^a\right\}>0,
    \qquad D_H<\delta_R^2/2
    \quad\Longrightarrow\quad \hat y_t(x;\boldsymbol{\theta})=c.
    \label{eq:main-stability-condition}
\end{equation}
The empty inner maximum is zero.
\end{theorem}
\noindent\textit{Proof.} See Appendix~\ref{app:response-details}.\hfill$\square$

This pointwise condition is not implied for every input by an average support budget.
Appendix~\ref{app:certificate-coverage} measures single-graph coverage.

\subsection{Bounding response drift at two scales}
\label{sec:response-budgets}

We measure the chosen response's change across $K$ nonempty support sources $\mathcal S_1,\ldots,\mathcal S_K$:
\begin{equation}
    \widehat D_t(\boldsymbol{\theta}^a,\boldsymbol{\theta})
    =\frac{1}{K}\sum_{r=1}^{K}\frac{1}{|\mathcal S_r|}
    \sum_{x\in\mathcal S_r}D_H(x;\boldsymbol{\theta}^a,\boldsymbol{\theta}).
    \label{eq:empirical-functional-divergence}
\end{equation}
Sources receive equal weight.
Small updates can accumulate drift, so we constrain change from both the current iterate and fixed task entry:
\begin{equation}
    \widehat D_t(\boldsymbol{\theta}_k,\boldsymbol{\theta}_{k+1})\leq\frac{\delta^2}{2},
    \qquad
    \widehat D_t(\boldsymbol{\theta}_0,\boldsymbol{\theta}_{k+1})\leq\frac{\rho^2}{2}.
    \label{eq:two-functional-budgets}
\end{equation}
The first budget limits one step; the second limits drift from $\boldsymbol{\theta}_0$, which resets at the next task.
Both references use the expanded candidates, including new-class competition before adaptation.
These are support-average constraints; conditional protection is weaker where $q_O^a$ is small.

After each task, sampled training features enter the $M=1024$-row support (Appendix~\ref{app:support-details}).

\subsection{Adapting within the response budgets}
\label{sec:feasible-updates}

Equal-length parameter steps can change the protected response by different amounts.
We therefore use the same response KL to define the local Fisher metric, with damping $\lambda>0$:
\begin{equation}
    \mathbf{F}_k=\left.\nabla_{\boldsymbol{\theta}}^2\widehat D_t(\boldsymbol{\theta}_k,\boldsymbol{\theta})\right|_{\boldsymbol{\theta}=\boldsymbol{\theta}_k},
    \qquad \mathbf{G}_k=\mathbf{F}_k+\lambda \mathbf{I},
    \qquad \|\mathbf{v}\|_{\mathbf{G}_k}^2=\mathbf{v}^\top \mathbf{G}_k\mathbf{v}.
    \label{eq:functional-fisher}
\end{equation}
This metric penalizes directions that rapidly change the response.
For the current cross-entropy gradient $\mathbf{g}_k$, implicit Fisher products and conjugate gradients approximate $\mathbf{G}_k\mathbf{v}_k=\mathbf{g}_k$ (Appendix~\ref{app:fisher-products}), yielding
\begin{equation}
    \mathbf{d}_k=-\delta\frac{\mathbf{v}_k}{\sqrt{\mathbf{g}_k^\top \mathbf{v}_k}}.
    \label{eq:natural-proposal}
\end{equation}
With an exact solve, this minimizes the linearized loss within the local metric ball.

Near the task-level KL boundary, an outward proposal may be almost eliminated by scalar shrinking even when tangent motion remains possible.
Let $\mathbf{a}_k=\nabla_{\boldsymbol{\theta}}\widehat D_t(\boldsymbol{\theta}_0,\boldsymbol{\theta}_k)$ and compute $\mathbf{u}_k\approx \mathbf{G}_k^{-1}\mathbf{a}_k$.
When the task divergence reaches $95\%$ of its budget and $\mathbf{a}_k^\top \mathbf{d}_k>0$, we form
\begin{equation}
    \mathbf{d}_k^{\mathrm{tan}}
    =\mathbf{d}_k-\frac{\mathbf{a}_k^\top \mathbf{d}_k}{\mathbf{a}_k^\top \mathbf{u}_k}\mathbf{u}_k.
    \label{eq:tangent-proposal}
\end{equation}
With an exact inverse, this $\mathbf{G}_k$-metric projection removes the first-order outward component and retains tangent motion.
A heuristic inward adjustment accommodates curvature; the correction is rescaled and kept only if nondegenerate and descending (Appendix~\ref{app:solver-details}).
Finite CG supplies an approximate proposal, while the original nonlinear KL determines feasibility: scalar searches shrink it as needed, and both budgets are checked before acceptance.

We use $\delta=0.002$, $\rho=0.1$, and $\lambda=0.01$.
At task one, $\mathbf{r}_1=(1)$ and $\mathbf{G}_k=\lambda \mathbf{I}$; the update is a normalized gradient step, so complete-system gains also include first-task optimization (Appendix~\ref{app:first-stage-attribution}).
Inference uses Equation~\eqref{eq:global-prediction} without the support or solver.

\section{Experiments}
\label{sec:experiments}

\subsection{Experimental Setup}

\paragraph{Datasets and baselines.}
We evaluate three class-incremental settings (Table~\ref{tab:cross-graph}).
The first uses four datasets---Cora, CiteSeer, WikiCS, and Photo---in four tasks; the second adds Sports-Fitness and Actor for six datasets and six tasks.
In both, each dataset forms one task.
The third uses the same four datasets as the first, split into two, two, three, and three tasks, respectively, for ten tasks. Fixed graphs provide transductive context; training labels arrive task by task.
We compare readout regularizers, prototype methods, LoRA-based adapters~\citep{hu2022lora}, and graph--text systems.
SimpleCIL uses BERT as its text backbone.
HiRP, SGD, and LwF share frozen GraphCLIP features and the pretrained readout; all use three epochs per task, batch size 16, and label smoothing 0.2.
HiRP uses at most 1024 support features and the default budgets in Section~\ref{sec:feasible-updates}.
Datasets and baseline settings are detailed in Appendices~\ref{app:experimental-details}--\ref{app:comparison-contract}.

\paragraph{Evaluation protocol.}
Training uses the complete prescribed split; prediction considers all arrived classes without task identities.
Let $A_{t,i}$ be accuracy on the test nodes of task $i$ after learning task $t$, and $n_i$ their number.
We report final accuracy (FA), final task-averaged accuracy (AA), and signed average forgetting (AF):
\begin{equation}
    \mathrm{FA}=\frac{\sum_{i=1}^{T}n_iA_{T,i}}{\sum_{i=1}^{T}n_i},
    \qquad \mathrm{AA}=\frac{1}{T}\sum_{i=1}^{T}A_{T,i},
    \qquad \mathrm{AF}=\frac{1}{T-1}\sum_{i=1}^{T-1}(A_{T,i}-A_{i,i}).
    \label{eq:evaluation-metrics}
\end{equation}
FA weights test nodes equally; AA weights tasks equally. Signed AF is the change since acquisition in percentage points; higher is better.
We report means and sample standard deviations from five runs for the main table and three for aggregate auxiliary experiments. Individual-node and boundary diagnostics, common-start full-sequence controls, History trajectories, and parameter sweeps use one run, with fixed splits and class order.
Bold marks the best mean; underlining marks the second best in Table~\ref{tab:cross-graph}.

\begin{table}[t]
\centering
\caption{Class-incremental classification over three dataset/task schedules.
Means are shown with sample standard deviations in subscript.
FA and AA are accuracy (\%); signed AF is forgetting (percentage points), with higher values better.
Bold and underline mark the best and second-best means, respectively. GraphCLIP denotes sequential joint fine-tuning;
backbones and $^{\dagger}$ adaptations are specified in Appendix~\ref{app:baseline-settings}.}
\label{tab:cross-graph}
\begingroup
\def\sminus{\mathord{\scriptstyle -}}
\fontsize{8}{9.5}\selectfont
\setlength{\tabcolsep}{2pt}
\renewcommand{\arraystretch}{1.2}
\resizebox{\linewidth}{!}{%
\begin{tabular}{@{}l@{\hspace{6pt}}ccc@{\hspace{8pt}}ccc@{\hspace{8pt}}ccc@{}}
\toprule
Method
& \multicolumn{3}{c}{4 datasets / 4 tasks}
& \multicolumn{3}{c}{6 datasets / 6 tasks}
& \multicolumn{3}{c}{4 datasets / 10 tasks} \\
\cmidrule(lr){2-4}\cmidrule(lr){5-7}\cmidrule(lr){8-10}
& FA $\uparrow$ & AA $\uparrow$ & AF $\uparrow$
& FA $\uparrow$ & AA $\uparrow$ & AF $\uparrow$
& FA $\uparrow$ & AA $\uparrow$ & AF $\uparrow$ \\
\midrule
GCN
& $50.30_{\pm0.06}$ & $21.71_{\pm0.03}$ & $\sminus81.26_{\pm0.45}$
& $1.00_{\pm0.82}$ & $11.99_{\pm0.24}$ & $\sminus84.37_{\pm0.30}$
& $14.08_{\pm0.02}$ & $9.69_{\pm0.01}$ & $\sminus90.58_{\pm0.32}$ \\
SGD
& $75.79_{\pm0.12}$ & $58.36_{\pm0.20}$ & $\sminus25.73_{\pm0.17}$
& $83.13_{\pm0.19}$ & $58.89_{\pm0.27}$ & $\sminus22.64_{\pm0.20}$
& $50.24_{\pm0.16}$ & $50.05_{\pm0.17}$ & $\sminus41.57_{\pm0.20}$ \\
LwF
& $\underline{76.60}_{\pm0.14}$ & $58.46_{\pm0.56}$ & $\sminus25.21_{\pm0.73}$
& $\underline{84.23}_{\pm0.14}$ & $61.63_{\pm0.22}$ & $\sminus19.03_{\pm0.22}$
& $56.70_{\pm0.63}$ & $49.87_{\pm0.39}$ & $\sminus40.99_{\pm0.43}$ \\
EWC
& $75.84_{\pm0.00}$ & $58.13_{\pm0.00}$ & $\sminus25.95_{\pm0.00}$
& $83.29_{\pm0.00}$ & $59.08_{\pm0.00}$ & $\sminus22.44_{\pm0.00}$
& $50.30_{\pm0.00}$ & $50.12_{\pm0.00}$ & $\sminus41.36_{\pm0.00}$ \\
MAS
& $66.09_{\pm0.84}$ & $61.11_{\pm0.56}$ & $\sminus12.26_{\pm1.77}$
& $66.85_{\pm0.00}$ & $62.79_{\pm0.00}$ & $\sminus7.79_{\pm0.00}$
& $57.74_{\pm1.10}$ & $54.27_{\pm1.12}$ & $\sminus31.22_{\pm2.40}$ \\
OGD
& $76.44_{\pm0.00}$ & $58.38_{\pm0.00}$ & $\sminus25.64_{\pm0.00}$
& $83.48_{\pm0.00}$ & $59.33_{\pm0.00}$ & $\sminus22.18_{\pm0.00}$
& $47.50_{\pm0.00}$ & $48.14_{\pm0.00}$ & $\sminus43.31_{\pm0.00}$ \\
Cosine
& $62.95_{\pm0.83}$ & $60.02_{\pm0.70}$ & $\sminus12.20_{\pm1.16}$
& $58.21_{\pm3.35}$ & $58.09_{\pm0.38}$ & $\sminus7.84_{\pm1.12}$
& $\underline{61.87}_{\pm1.49}$ & $\underline{61.65}_{\pm1.08}$ & $\sminus12.48_{\pm0.78}$ \\
TEEN
& $44.30_{\pm2.94}$ & $40.53_{\pm2.01}$ & $\sminus21.40_{\pm0.66}$
& $45.51_{\pm6.53}$ & $41.54_{\pm2.88}$ & $\sminus13.25_{\pm0.11}$
& $50.29_{\pm4.42}$ & $47.45_{\pm2.20}$ & $\sminus18.46_{\pm1.66}$ \\
TPP-style
& $72.18_{\pm0.11}$ & $\underline{69.35}_{\pm0.23}$ & $\underline{\sminus2.29}_{\pm0.03}$
& $76.04_{\pm0.04}$ & $\underline{70.94}_{\pm0.06}$ & $\underline{\sminus1.40}_{\pm0.02}$
& $54.86_{\pm0.50}$ & $61.47_{\pm0.72}$ & $\underline{\sminus9.83}_{\pm0.38}$ \\
\midrule
BERT
& $44.27_{\pm0.12}$ & $19.11_{\pm0.05}$ & $\sminus75.19_{\pm0.51}$
& $34.02_{\pm8.56}$ & $19.03_{\pm1.52}$ & $\sminus70.49_{\pm2.15}$
& $12.97_{\pm0.08}$ & $8.93_{\pm0.05}$ & $\sminus83.86_{\pm0.20}$ \\
RoBERTa
& $46.14_{\pm0.34}$ & $19.96_{\pm0.23}$ & $\sminus77.38_{\pm0.52}$
& $75.12_{\pm1.17}$ & $27.64_{\pm0.51}$ & $\sminus62.55_{\pm0.62}$
& $13.21_{\pm0.04}$ & $9.09_{\pm0.03}$ & $\sminus83.57_{\pm1.40}$ \\
SimpleCIL
& $42.47_{\pm0.93}$ & $52.15_{\pm0.84}$ & $\mathbf{\sminus1.79}_{\pm0.12}$
& $37.66_{\pm1.51}$ & $47.21_{\pm1.60}$ & $\mathbf{\sminus1.20}_{\pm0.18}$
& $42.09_{\pm1.48}$ & $49.61_{\pm1.36}$ & $\mathbf{\sminus8.19}_{\pm0.58}$ \\
\midrule
GCN\textsubscript{LLMEmb}
& $50.39_{\pm0.40}$ & $21.75_{\pm0.17}$ & $\sminus81.82_{\pm0.24}$
& $0.66_{\pm0.21}$ & $12.21_{\pm0.20}$ & $\sminus84.62_{\pm0.36}$
& $14.08_{\pm0.02}$ & $9.69_{\pm0.01}$ & $\sminus91.17_{\pm0.17}$ \\
GraphCLIP
& $44.27_{\pm1.77}$ & $19.12_{\pm0.78}$ & $\sminus74.64_{\pm0.95}$
& $19.95_{\pm9.81}$ & $4.68_{\pm2.15}$ & $\sminus70.40_{\pm3.20}$
& $15.37_{\pm1.99}$ & $9.88_{\pm1.06}$ & $\sminus78.03_{\pm0.70}$ \\
GraphGPT$^{\dagger}$
& $72.87_{\pm0.94}$ & $52.05_{\pm0.86}$ & $\sminus37.19_{\pm1.24}$
& $83.51_{\pm0.83}$ & $47.86_{\pm3.59}$ & $\sminus40.25_{\pm4.42}$
& $26.72_{\pm2.19}$ & $24.84_{\pm1.74}$ & $\sminus71.19_{\pm1.69}$ \\
LLaGA$^{\dagger}$
& $49.42_{\pm0.14}$ & $21.34_{\pm0.08}$ & $\sminus77.58_{\pm1.02}$
& $21.35_{\pm13.00}$ & $15.34_{\pm2.62}$ & $\sminus76.88_{\pm3.19}$
& $13.99_{\pm0.05}$ & $9.63_{\pm0.04}$ & $\sminus87.49_{\pm0.87}$ \\
\midrule
LoRA
& $76.02_{\pm0.15}$ & $60.87_{\pm0.13}$ & $\sminus20.26_{\pm0.22}$
& $83.63_{\pm0.07}$ & $64.07_{\pm0.16}$ & $\sminus14.61_{\pm0.14}$
& $51.18_{\pm0.12}$ & $51.11_{\pm0.15}$ & $\sminus38.60_{\pm0.19}$ \\
InfLoRA$^{\dagger}$
& $54.54_{\pm0.43}$ & $30.35_{\pm0.94}$ & $\sminus41.32_{\pm2.40}$
& $44.40_{\pm3.99}$ & $26.12_{\pm1.07}$ & $\sminus27.31_{\pm1.04}$
& $47.48_{\pm1.41}$ & $37.35_{\pm1.95}$ & $\sminus26.86_{\pm6.04}$ \\
SD-LoRA
& $62.15_{\pm1.37}$ & $41.41_{\pm1.60}$ & $\sminus27.13_{\pm1.67}$
& $65.82_{\pm2.04}$ & $38.03_{\pm1.10}$ & $\sminus18.12_{\pm1.00}$
& $47.41_{\pm0.53}$ & $37.51_{\pm1.24}$ & $\sminus30.09_{\pm5.52}$ \\
G2LoRA$^{\dagger}$
& $72.70_{\pm0.47}$ & $57.90_{\pm0.43}$ & $\sminus20.53_{\pm0.43}$
& $81.07_{\pm0.05}$ & $59.60_{\pm0.13}$ & $\sminus16.07_{\pm0.19}$
& $49.05_{\pm0.07}$ & $51.46_{\pm0.13}$ & $\sminus32.43_{\pm0.16}$ \\
\midrule
HiRP (ours)
& $\mathbf{80.05}_{\pm0.15}$ & $\mathbf{71.19}_{\pm0.22}$ & $\sminus8.51_{\pm0.36}$
& $\mathbf{87.27}_{\pm0.44}$ & $\mathbf{73.45}_{\pm0.40}$ & $\sminus6.11_{\pm0.44}$
& $\mathbf{75.30}_{\pm0.31}$ & $\mathbf{69.60}_{\pm0.49}$ & $\sminus14.88_{\pm0.30}$ \\
\bottomrule
\end{tabular}%
}
\endgroup
\end{table}

\subsection{Main Results on Class-Incremental Learning}

Table~\ref{tab:cross-graph} shows that HiRP leads in FA and AA across all settings. Its gains over the strongest baseline are 3.45, 3.04, and 13.43 points in FA, and 1.84, 2.51, and 7.95 points in AA, respectively; the ten-task setting shows the largest gains after repeated class expansion.

HiRP also has better AF than SGD and LwF. SimpleCIL forgets less but has substantially lower FA and AA, so lower forgetting alone does not imply better continual performance. Appendix~\ref{app:baseline-retention} separates acquisition and retention; Section~\ref{sec:case-study} compares common-start updates, and Appendix~\ref{app:memory-distillation} compares distillation under the same memory budget.

\subsection{Classification Throughout the Stream}
\label{sec:stagewise-performance}

\begin{figure}[t]
    \centering
    \includegraphics[width=\linewidth]{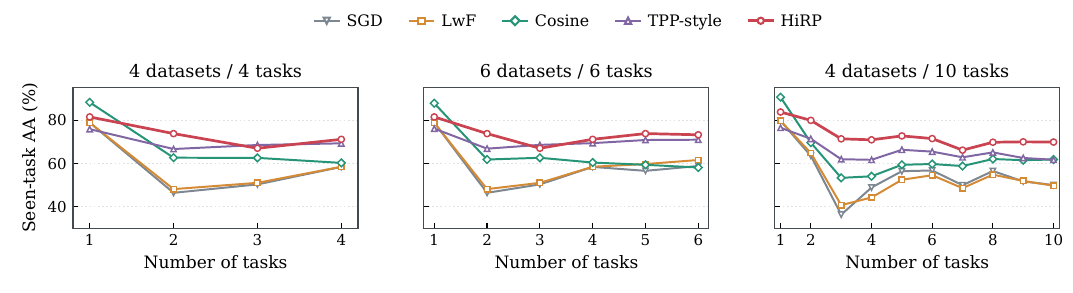}
    \caption{Accuracy throughout learning, averaged over arrived tasks. Each point uses all classes introduced up to that task.}
    \label{fig:stagewise-accuracy}
\end{figure}

HiRP stays above SGD and LwF throughout all three streams (Figure~\ref{fig:stagewise-accuracy}). At task 2 of the four- and six-task settings, AA is $73.77\%$ versus LwF's $48.07\%$; the ten-task advantage persists through repeated arrivals. These curves track acquisition and retention jointly.
Appendices~\ref{app:single-graph-results}, \ref{app:comparison-contract}, and~\ref{app:hyperparameter-sensitivity} report single-graph results, training costs, and parameter sensitivity.

\subsection{Zero-Shot Transfer Retention}
\label{sec:zero-shot-retention}
\begin{wrapfigure}{R}{0.43\textwidth}
    \centering
    \vspace{-8pt}
    \includegraphics[width=\linewidth]
        {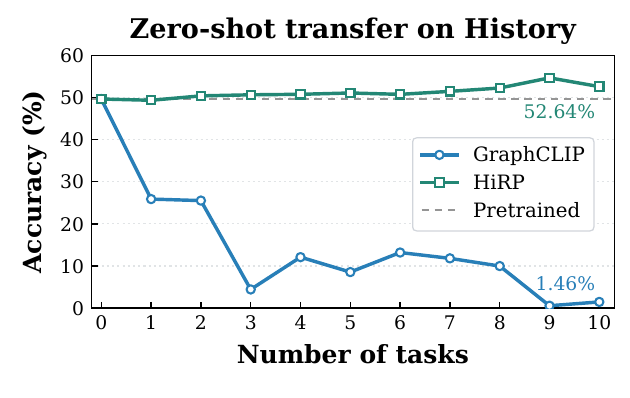}
    \caption{Zero-shot accuracy of GraphCLIP and HiRP on History during class-incremental fine-tuning.}
    \label{fig:zero-shot-retention}
    \vspace{-10pt}
\end{wrapfigure}

We revisit the History transfer degradation in Figure~\ref{fig:zero_shot_degradation}, comparing HiRP and GraphCLIP~\citep{zhu2025graphclip} fine-tuning over ten tasks from the same pretrained model. History is excluded from training and feature support; evaluation uses its complete class set without target-domain adaptation (Appendix~\ref{app:baseline-settings}). HiRP finishes at $52.64\%$, above the initial $49.69\%$, while GraphCLIP falls to $1.46\%$ (Figure~\ref{fig:zero-shot-retention}). HiRP thus retains zero-shot transfer while learning new classes; in a common-start readout comparison, it also exceeds SGD's mean across three held-out domains (Table~\ref{tab:common-start-transfer}).

\subsection{Ablation Studies}

\begin{table}[t]
\centering
\caption{Response and update-direction ablation under the same
data, support rules, and numerical KL budgets.
The mass response keeps only historical/current group probabilities.
HiRP combines the hierarchical response with tangent correction.}
\label{tab:ablation}
\begingroup
\def\sminus{\mathord{\scriptstyle -}}
\fontsize{8}{9.5}\selectfont
\setlength{\tabcolsep}{2pt}
\renewcommand{\arraystretch}{1.2}
\resizebox{\linewidth}{!}{%
\begin{tabular}{
    @{}l
    @{\hspace{6pt}}ccc
    @{\hspace{8pt}}ccc
    @{\hspace{8pt}}ccc
    @{}
}
\toprule
Method
& \multicolumn{3}{c}{4 datasets / 4 tasks}
& \multicolumn{3}{c}{6 datasets / 6 tasks}
& \multicolumn{3}{c}{4 datasets / 10 tasks} \\
\cmidrule(lr){2-4}
\cmidrule(lr){5-7}
\cmidrule(lr){8-10}
& FA $\uparrow$ & AA $\uparrow$ & AF $\uparrow$
& FA $\uparrow$ & AA $\uparrow$ & AF $\uparrow$
& FA $\uparrow$ & AA $\uparrow$ & AF $\uparrow$ \\
\midrule
Mass + radial
& $77.10_{\pm0.40}$ & $67.23_{\pm0.45}$ & $\sminus13.91_{\pm0.83}$
& $84.12_{\pm1.09}$ & $69.60_{\pm0.41}$ & $\sminus10.38_{\pm0.26}$
& $64.49_{\pm1.37}$ & $63.07_{\pm1.20}$ & $\sminus25.38_{\pm1.13}$ \\
Mass + tangent
& $78.50_{\pm0.11}$ & $68.07_{\pm1.20}$ & $\sminus14.07_{\pm2.12}$
& $85.82_{\pm0.54}$ & $68.18_{\pm0.63}$ & $\sminus13.24_{\pm0.63}$
& $62.31_{\pm1.58}$ & $58.06_{\pm1.84}$ & $\sminus31.27_{\pm2.04}$ \\
Hier. + radial
& $75.52_{\pm0.88}$ & $69.03_{\pm0.14}$ & $\mathbf{\sminus7.79}_{\pm0.41}$
& $82.83_{\pm0.47}$ & $71.59_{\pm0.14}$ & $\mathbf{\sminus5.17}_{\pm0.17}$
& $74.36_{\pm0.11}$ & $69.71_{\pm0.15}$ & $\mathbf{\sminus12.76}_{\pm0.27}$ \\
\midrule
HiRP (ours)
& $\mathbf{80.09}_{\pm0.19}$ & $\mathbf{71.12}_{\pm0.26}$ & $\sminus8.43_{\pm0.32}$
& $\mathbf{87.04}_{\pm0.42}$ & $\mathbf{73.22}_{\pm0.23}$ & $\sminus6.21_{\pm0.32}$
& $\mathbf{75.43}_{\pm0.05}$ & $\mathbf{69.86}_{\pm0.13}$ & $\sminus14.74_{\pm0.15}$ \\
\bottomrule
\end{tabular}%
}
\endgroup
\end{table}

We ablate the response and update direction. \textbf{Mass} preserves aggregate historical/current probabilities, while \textbf{Hier.} also preserves within-history distinctions. \textbf{Radial} shrinks the update; \textbf{tangent} corrects its direction. Table~\ref{tab:ablation} compares Mass + radial, Mass + tangent, and Hier. + radial with full HiRP (Hier. + tangent).

Replacing the hierarchical response with Mass while retaining tangent correction reduces FA, AA, and AF across all three settings (Table~\ref{tab:ablation}). Removing tangent correction lowers FA and AA but improves AF, trading adaptation for retention. Tangent correction alone does not consistently improve Mass, so response structure and update direction must be designed jointly.

\subsection{Case Study}
\label{sec:case-study}

\paragraph{Preserving historical decisions.}
From a shared task-5 checkpoint in the ten-task setting, HiRP, Mass, and SGD take 32 identical batches; SGD* matches HiRP's step norms. For SGD*, Mass, and HiRP, old-only accuracy changes by $-2.60$, $-0.62$, and $+0.11$ points, while old-to-new crossing rates are $2.63\%$, $0.54\%$, and $0.25\%$ (Figure~\ref{fig:case-preservation}a,b). New-class accuracy differs by less than $0.16$ points: HiRP reduces crossings while preserving historical distinctions.

\begin{figure}[!htbp]
    \centering
    \includegraphics[width=\linewidth]{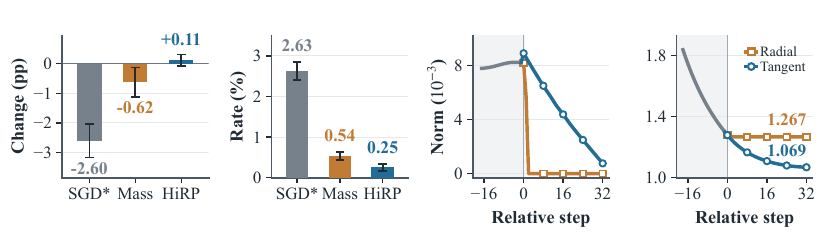}\par
    \vspace{-1mm}
    \noindent\makebox[\linewidth][l]{%
        \hspace*{.164\linewidth}\makebox[0pt][c]{\footnotesize\bfseries (a) Historical accuracy}%
        \hspace*{.248\linewidth}\makebox[0pt][c]{\footnotesize\bfseries (b) Old-to-new crossings}%
        \hspace*{.244\linewidth}\makebox[0pt][c]{\footnotesize\bfseries (c) Update norm}%
        \hspace*{.247\linewidth}\makebox[0pt][c]{\footnotesize\bfseries (d) Training loss}%
    }\par
    \caption{Preserving historical decisions and maintaining adaptation.
    (a,b) Ten-task setting, Task 5: changes after 32 common-start updates; SGD* matches HiRP's step norms. Error bars show sample standard deviations.
    (c,d) CiteSeer, Task 2: accepted update norms and training loss under the same decreasing step-radius budget. Gray denotes the shared initial trajectory.}
    \label{fig:case-preservation}
    \label{fig:case-boundary-learning}
\end{figure}

\paragraph{Learning near the protection boundary.}
We revisit the fixed-minibatch trajectory in Figure~\ref{fig:motivation-radial-stalling} from the same state and budget schedule. As the step-radius budget decreases, tangent correction lowers training loss to $1.069$ at relative step 32, while radial shrinking stalls at $1.267$ (Figure~\ref{fig:case-boundary-learning}c,d). Tangent updates become smaller yet continue learning near the protection boundary. In separate complete task-2 runs, new-class accuracy rises from $64.47\%$ to $66.67\%$, with a $0.31$-point old-class decrease. Appendices~\ref{app:boundary-trajectories} and~\ref{app:case-study} give diagnostic settings and single-step results.

\section{Conclusion}
\label{sec:conclusion}

We studied how pretrained graph--text models can learn classes arriving across graph domains while retaining historical performance and zero-shot transfer. Our findings show that preserving distinctions among historical classes alone is insufficient as new candidates enter the prediction space: competition between historical and new classes also matters. HiRP captures these relationships in a hierarchical response and uses it to constrain and guide readout updates. Across three class-incremental streams, HiRP improves unified classification and mitigates zero-shot degradation on held-out domains. These results suggest that continual adaptation benefits from identifying which predictive relationships to preserve and updating the model within the resulting constraints.

\subsection*{AI use statement}

In this work, we used generative AI tools for writing assistance and research execution.
For writing assistance, these tools supported editing and polishing manuscript text and \LaTeX{} files, as well as translating text.
For research execution, they assisted with interpreting experimental results, revising scientific figures and plotting code, and debugging experimental code.
The authors manually reviewed all AI-assisted content for correctness and consistency with the research and take full responsibility for the final content of this paper.

\subsection*{Reproducibility statement}

The learning and inference protocol is defined in Section~\ref{sec:setup}, and the experimental setup is given in Section~\ref{sec:experiments}. Appendix~\ref{app:method-details} provides derivations and implementation details for HiRP. Appendix~\ref{app:experimental-details} documents the data and task protocols, baseline implementations, comparison conditions, and additional results, including hyperparameter sensitivity.

\clearpage
\bibliography{ref}

@inproceedings{aljundi2018mas,
  title = {Memory Aware Synapses: Learning what (not) to forget},
  author = {Aljundi, Rahaf and Babiloni, Francesca and Elhoseiny, Mohamed and Rohrbach, Marcus and Tuytelaars, Tinne},
  booktitle = {Computer Vision -- {ECCV} 2018},
  series = {Lecture Notes in Computer Science},
  volume = {11207},
  pages = {144--161},
  year = {2018},
  publisher = {Springer International Publishing},
  doi = {10.1007/978-3-030-01219-9_9},
  url = {https://doi.org/10.1007/978-3-030-01219-9_9}
}

@inproceedings{farajtabar2020ogd,
  title = {Orthogonal Gradient Descent for Continual Learning},
  author = {Farajtabar, Mehrdad and Azizan, Navid and Mott, Alex and Li, Ang},
  booktitle = {Proceedings of the Twenty Third International Conference on Artificial Intelligence and Statistics},
  series = {Proceedings of Machine Learning Research},
  volume = {108},
  pages = {3762--3773},
  year = {2020},
  publisher = {PMLR},
  url = {https://proceedings.mlr.press/v108/farajtabar20a.html}
}

@article{cheng2025llm4gcl,
  title = {Can {LLMs} Alleviate Catastrophic Forgetting in Graph Continual Learning? {A} Systematic Study},
  author = {Cheng, Ziyang and Li, Zhixun and Li, Yuhan and Song, Yixin and Zhao, Kangyi and Cheng, Dawei and Li, Jia and Cheng, Hong and Yu, Jeffrey Xu},
  journal = {arXiv preprint arXiv:2505.18697v2},
  year = {2025},
  url = {https://arxiv.org/abs/2505.18697v2}
}

@inproceedings{liu2026adaligner,
  title = {Learning Noise-Resilient and Transferable Graph-Text Alignment via Dynamic Quality Assessment},
  author = {Liu, Yuhang and Shao, Minglai and Wo, Zengyi and Chu, Yunlong and Hao, Bing and Liu, Shengzhong and Wang, Ruijie and Li, Jianxin},
  booktitle = {Proceedings of the 49th International ACM SIGIR Conference on Research and Development in Information Retrieval},
  pages = {1209--1220},
  year = {2026},
  publisher = {ACM},
  doi = {10.1145/3805712.3809579},
  url = {https://doi.org/10.1145/3805712.3809579}
}

@inproceedings{zhu2025graphclip,
  title = {{GraphCLIP}: Enhancing Transferability in Graph Foundation Models for Text-Attributed Graphs},
  author = {Zhu, Yun and Shi, Haizhou and Wang, Xiaotang and Liu, Yongchao and Wang, Yaoke and Peng, Boci and Hong, Chuntao and Tang, Siliang},
  booktitle = {Proceedings of the ACM on Web Conference 2025},
  pages = {2183--2197},
  year = {2025},
  publisher = {ACM},
  doi = {10.1145/3696410.3714801},
  url = {https://doi.org/10.1145/3696410.3714801}
}

@inproceedings{rebuffi2017icarl,
  title = {{iCaRL}: Incremental Classifier and Representation Learning},
  author = {Rebuffi, Sylvestre-Alvise and Kolesnikov, Alexander and Sperl, Georg and Lampert, Christoph H.},
  booktitle = {Proceedings of the IEEE Conference on Computer Vision and Pattern Recognition},
  year = {2017},
  pages = {5533--5542},
  doi = {10.1109/CVPR.2017.587},
  url = {https://doi.org/10.1109/CVPR.2017.587}
}

@inproceedings{hou2019lucir,
  title = {Learning a Unified Classifier Incrementally via Rebalancing},
  author = {Hou, Saihui and Pan, Xinyu and Loy, Chen Change and Wang, Zilei and Lin, Dahua},
  booktitle = {Proceedings of the IEEE/CVF Conference on Computer Vision and Pattern Recognition},
  year = {2019},
  pages = {831--839},
  doi = {10.1109/CVPR.2019.00092},
  url = {https://openaccess.thecvf.com/content_CVPR_2019/html/Hou_Learning_a_Unified_Classifier_Incrementally_via_Rebalancing_CVPR_2019_paper.html}
}

@inproceedings{titsias2020functional,
  title = {Functional Regularisation for Continual Learning with {Gaussian} Processes},
  author = {Titsias, Michalis K. and Schwarz, Jonathan and Matthews, Alexander G. de G. and Pascanu, Razvan and Teh, Yee Whye},
  booktitle = {International Conference on Learning Representations},
  year = {2020},
  url = {https://iclr.cc/virtual/2020/poster/1650}
}

@inproceedings{pan2020fromp,
  title = {Continual Deep Learning by Functional Regularisation of Memorable Past},
  author = {Pan, Pingbo and Swaroop, Siddharth and Immer, Alexander and Eschenhagen, Runa and Turner, Richard E. and Khan, Mohammad Emtiyaz},
  booktitle = {Advances in Neural Information Processing Systems},
  volume = {33},
  pages = {4453--4464},
  year = {2020},
  url = {https://proceedings.neurips.cc/paper/2020/hash/2f3bbb9730639e9ea48f309d9a79ff01-Abstract.html}
}

@inproceedings{buzzega2020dark,
  title = {Dark Experience for General Continual Learning: A Strong, Simple Baseline},
  author = {Buzzega, Pietro and Boschini, Matteo and Porrello, Angelo and Abati, Davide and Calderara, Simone},
  booktitle = {Advances in Neural Information Processing Systems},
  volume = {33},
  pages = {15920--15930},
  year = {2020},
  url = {https://proceedings.neurips.cc/paper/2020/hash/b704ea2c39778f07c617f6b7ce480e9e-Abstract.html}
}

@article{kirkpatrick2017ewc,
  title = {Overcoming Catastrophic Forgetting in Neural Networks},
  author = {Kirkpatrick, James and Pascanu, Razvan and Rabinowitz, Neil and Veness, Joel and Desjardins, Guillaume and Rusu, Andrei A. and Milan, Kieran and Quan, John and Ramalho, Tiago and Grabska-Barwinska, Agnieszka and Hassabis, Demis and Clopath, Claudia and Kumaran, Dharshan and Hadsell, Raia},
  journal = {Proceedings of the National Academy of Sciences},
  volume = {114},
  number = {13},
  pages = {3521--3526},
  year = {2017},
  doi = {10.1073/pnas.1611835114},
  url = {https://doi.org/10.1073/pnas.1611835114}
}

@inproceedings{wang2022l2p,
  title = {Learning to Prompt for Continual Learning},
  author = {Wang, Zifeng and Zhang, Zizhao and Lee, Chen-Yu and Zhang, Han and Sun, Ruoxi and Ren, Xiaoqi and Su, Guolong and Perot, Vincent and Dy, Jennifer and Pfister, Tomas},
  booktitle = {Proceedings of the IEEE/CVF Conference on Computer Vision and Pattern Recognition},
  year = {2022},
  pages = {139--149},
  doi = {10.1109/CVPR52688.2022.00024},
  url = {https://openaccess.thecvf.com/content/CVPR2022/html/Wang_Learning_To_Prompt_for_Continual_Learning_CVPR_2022_paper.html}
}

@inproceedings{niu2024tpp,
  title = {Replay-and-Forget-Free Graph Class-Incremental Learning: A Task Profiling and Prompting Approach},
  author = {Niu, Chaoxi and Pang, Guansong and Chen, Ling and Liu, Bing},
  booktitle = {Advances in Neural Information Processing Systems},
  volume = {37},
  pages = {87978--88002},
  year = {2024},
  doi = {10.52202/079017-2793},
  url = {https://proceedings.neurips.cc/paper_files/paper/2024/hash/a07e87ecfa8a651d62257571669b0150-Abstract-Conference.html}
}

@inproceedings{wang2026g2lora,
  title = {{$G^2$LoRA}: Gradient Orthogonal Low-Rank Adaptation Framework for Graph Continual Learning on Text-Attributed Graphs},
  author = {Wang, Yuhan and Ding, Yibo and Ye, Yutong and Zhao, Mufan and Zhang, Wenbo and Wang, Ruijie and Li, Jianxin},
  booktitle = {Proceedings of the 32nd ACM SIGKDD Conference on Knowledge Discovery and Data Mining V.2},
  pages = {5184--5195},
  year = {2026},
  publisher = {ACM},
  doi = {10.1145/3770855.3817966},
  url = {https://doi.org/10.1145/3770855.3817966}
}

@inproceedings{li2016learning,
  title = {Learning without Forgetting},
  author = {Li, Zhizhong and Hoiem, Derek},
  booktitle = {European Conference on Computer Vision},
  series = {Lecture Notes in Computer Science},
  volume = {9908},
  pages = {614--629},
  year = {2016},
  publisher = {Springer International Publishing},
  doi = {10.1007/978-3-319-46493-0_37},
  url = {https://doi.org/10.1007/978-3-319-46493-0_37}
}

@inproceedings{zhao2022dkd,
  title = {Decoupled Knowledge Distillation},
  author = {Zhao, Borui and Cui, Quan and Song, Renjie and Qiu, Yiyu and Liang, Jiajun},
  booktitle = {Proceedings of the IEEE/CVF Conference on Computer Vision and Pattern Recognition},
  pages = {11943--11952},
  year = {2022},
  doi = {10.1109/CVPR52688.2022.01165},
  url = {https://doi.org/10.1109/CVPR52688.2022.01165}
}

@inproceedings{schulman2015trpo,
  title = {Trust Region Policy Optimization},
  author = {Schulman, John and Levine, Sergey and Abbeel, Pieter and Jordan, Michael and Moritz, Philipp},
  booktitle = {Proceedings of the International Conference on Machine Learning},
  year = {2015},
  pages = {1889--1897},
  url = {https://proceedings.mlr.press/v37/schulman15.html}
}

@inproceedings{kao2021ncl,
  title = {Natural Continual Learning: Success Is a Journey, Not (Just) a Destination},
  author = {Kao, Ta-Chu and Jensen, Kristopher T. and van de Ven, Gido M. and Bernacchia, Alberto and Hennequin, Guillaume},
  booktitle = {Advances in Neural Information Processing Systems},
  volume = {34},
  pages = {28067--28079},
  year = {2021},
  url = {https://proceedings.neurips.cc/paper/2021/hash/ec5aa0b7846082a2415f0902f0da88f2-Abstract.html}
}

@inproceedings{lin2022trgp,
  title = {{TRGP}: Trust Region Gradient Projection for Continual Learning},
  author = {Lin, Sen and Yang, Li and Fan, Deliang and Zhang, Junshan},
  booktitle = {International Conference on Learning Representations},
  year = {2022},
  url = {https://iclr.cc/virtual/2022/poster/6867}
}

@article{wang2026trcl,
  title = {Trust Region Continual Learning as an Implicit Meta-Learner},
  author = {Wang, Zekun and Gupta, Anant and MacLellan, Christopher J.},
  journal = {arXiv preprint arXiv:2602.02417},
  year = {2026},
  url = {https://arxiv.org/abs/2602.02417}
}

@inproceedings{cermelli2020mib,
  title = {Modeling the Background for Incremental Learning in Semantic Segmentation},
  author = {Cermelli, Fabio and Mancini, Massimiliano and {Rota Bul\`o}, Samuel and Ricci, Elisa and Caputo, Barbara},
  booktitle = {Proceedings of the IEEE/CVF Conference on Computer Vision and Pattern Recognition},
  pages = {9230--9239},
  year = {2020},
  doi = {10.1109/CVPR42600.2020.00925},
  url = {https://doi.org/10.1109/CVPR42600.2020.00925}
}

@inproceedings{hu2022lora,
  title = {{LoRA}: Low-Rank Adaptation of Large Language Models},
  author = {Hu, Edward J. and Shen, Yelong and Wallis, Phillip and Allen-Zhu, Zeyuan and Li, Yuanzhi and Wang, Shean and Wang, Lu and Chen, Weizhu},
  booktitle = {International Conference on Learning Representations},
  year = {2022},
  url = {https://iclr.cc/virtual/2022/poster/6319}
}

@inproceedings{kipf2017gcn,
  title = {Semi-Supervised Classification with Graph Convolutional Networks},
  author = {Kipf, Thomas N. and Welling, Max},
  booktitle = {International Conference on Learning Representations},
  year = {2017},
  url = {https://openreview.net/pdf?id=SJU4ayYgl}
}

@inproceedings{wang2023teen,
  title = {Few-Shot Class-Incremental Learning via Training-Free Prototype Calibration},
  author = {Wang, Qi-Wei and Zhou, Da-Wei and Zhang, Yi-Kai and Zhan, De-Chuan and Ye, Han-Jia},
  booktitle = {Advances in Neural Information Processing Systems},
  volume = {36},
  pages = {15060--15076},
  year = {2023},
  doi = {10.52202/075280-0661},
  url = {https://proceedings.neurips.cc/paper_files/paper/2023/hash/30dfe47a3ccbee68cffa0c19ccb1bc00-Abstract-Conference.html}
}

@inproceedings{devlin2019bert,
  title = {{BERT}: Pre-training of Deep Bidirectional Transformers for Language Understanding},
  author = {Devlin, Jacob and Chang, Ming-Wei and Lee, Kenton and Toutanova, Kristina},
  booktitle = {Proceedings of the 2019 Conference of the North American Chapter of the Association for Computational Linguistics: Human Language Technologies, Volume 1 (Long and Short Papers)},
  year = {2019},
  pages = {4171--4186},
  publisher = {Association for Computational Linguistics},
  doi = {10.18653/v1/N19-1423},
  url = {https://aclanthology.org/N19-1423/}
}

@article{zhou2025simplecil,
  title = {Revisiting Class-Incremental Learning with Pre-Trained Models: Generalizability and Adaptivity Are All You Need},
  author = {Zhou, Da-Wei and Cai, Zi-Wen and Ye, Han-Jia and Zhan, De-Chuan and Liu, Ziwei},
  journal = {International Journal of Computer Vision},
  volume = {133},
  number = {3},
  pages = {1012--1032},
  year = {2025},
  doi = {10.1007/s11263-024-02218-0},
  url = {https://doi.org/10.1007/s11263-024-02218-0}
}

@inproceedings{tang2024graphgpt,
  title = {{GraphGPT}: Graph Instruction Tuning for Large Language Models},
  author = {Tang, Jiabin and Yang, Yuhao and Wei, Wei and Shi, Lei and Su, Lixin and Cheng, Suqi and Yin, Dawei and Huang, Chao},
  booktitle = {Proceedings of the 47th International ACM SIGIR Conference on Research and Development in Information Retrieval},
  pages = {491--500},
  year = {2024},
  publisher = {ACM},
  doi = {10.1145/3626772.3657775},
  url = {https://doi.org/10.1145/3626772.3657775}
}

@inproceedings{liang2024inflora,
  title = {{InfLoRA}: Interference-Free Low-Rank Adaptation for Continual Learning},
  author = {Liang, Yan-Shuo and Li, Wu-Jun},
  booktitle = {Proceedings of the IEEE/CVF Conference on Computer Vision and Pattern Recognition},
  pages = {23638--23647},
  year = {2024},
  doi = {10.1109/CVPR52733.2024.02231},
  url = {https://openaccess.thecvf.com/content/CVPR2024/html/Liang_InfLoRA_Interference-Free_Low-Rank_Adaptation_for_Continual_Learning_CVPR_2024_paper.html}
}

@inproceedings{chen2024llaga,
  title = {{LLaGA}: Large Language and Graph Assistant},
  author = {Chen, Runjin and Zhao, Tong and Jaiswal, Ajay Kumar and Shah, Neil and Wang, Zhangyang},
  booktitle = {Proceedings of the 41st International Conference on Machine Learning},
  series = {Proceedings of Machine Learning Research},
  volume = {235},
  pages = {7809--7823},
  year = {2024},
  publisher = {PMLR},
  url = {https://proceedings.mlr.press/v235/chen24bh.html}
}

@article{liu2019roberta,
  title = {{RoBERTa}: A Robustly Optimized {BERT} Pretraining Approach},
  author = {Liu, Yinhan and Ott, Myle and Goyal, Naman and Du, Jingfei and Joshi, Mandar and Chen, Danqi and Levy, Omer and Lewis, Mike and Zettlemoyer, Luke and Stoyanov, Veselin},
  journal = {arXiv preprint arXiv:1907.11692},
  year = {2019},
  url = {https://arxiv.org/abs/1907.11692}
}

@inproceedings{wu2025sdlora,
  title = {{SD-LoRA}: Scalable Decoupled Low-Rank Adaptation for Class Incremental Learning},
  author = {Wu, Yichen and Piao, Hongming and Huang, Long-Kai and Wang, Renzhen and Li, Wanhua and Pfister, Hanspeter and Meng, Deyu and Ma, Kede and Wei, Ying},
  booktitle = {The Thirteenth International Conference on Learning Representations},
  year = {2025},
  url = {https://proceedings.iclr.cc/paper_files/paper/2025/hash/92f43b1d33fae4aa1958f75317f0cec1-Abstract-Conference.html}
}
\bibliographystyle{iclr2027_conference}

\appendix
\section{HiRP: Derivations and Implementation Details}
\label{app:method-details}

\subsection{Response design in context}
\label{app:response-comparison}

Table~\ref{tab:response-comparison} distinguishes the protected quantities and their input support.
HiRP organizes semantic competition by class arrival without sample-level historical labels: the frozen classifier scores current class descriptions at task entry.
Individual historical probabilities and collective current mass represent this structure, checked against current-iterate and task-start references.

\begin{table}[!htbp]
\centering
\caption{Response protection in related formulations. The comparison concerns the protected object and its inputs, not a claim that probability aggregation or functional regularization is unique to HiRP.}
\label{tab:response-comparison}
\small
\renewcommand{\arraystretch}{1.15}
\begin{tabular}{@{}p{.14\linewidth}p{.40\linewidth}p{.39\linewidth}@{}}
\toprule
Method & Protected quantity & Inputs and enforcement \\
\midrule
LwF & Historical-task outputs & Current-task inputs; distillation loss. \\
FRCL / FROMP & Function posteriors / past predictions & Inducing inputs / memorable examples; functional regularization. \\
MiB & Historical-class outputs and aggregated background/current-class probability & Segmentation inputs; background-aware distillation. \\
DKD & Target-versus-rest mass and non-target conditional distribution & Labeled distillation inputs; separately weighted losses. \\
\midrule
HiRP & Historical discrimination and aggregate group competition, represented by individual historical probabilities and current-group mass & Bounded unlabeled features; local and task-level KL budgets on the same expanded candidates, with checked updates. \\
\bottomrule
\end{tabular}
\end{table}

\subsection{Semantic competition and historical prediction errors}
\label{app:competition-geometry}

Fix a task $t>1$, a nonzero readout $\mathbf{W}\widetilde{\mathbf{h}}(x)$, and finite scores on the common candidate set $S=\mathcal O_t\cup\mathcal N_t$.
Write $O=\mathcal O_t$, $N=\mathcal N_t$, and $\boldsymbol{\nu}=\mathbf{W}\widetilde{\mathbf{h}}/\|\mathbf{W}\widetilde{\mathbf{h}}\|$.
The scale $\gamma>0$ and unit class vectors $\mathbf{u}_c$ are fixed.

\begin{theorem}[Semantic margins and predictive responses]
\label{prop:competition-geometry}
For every pair $i,j\in S$,
\begin{equation}
    \log\frac{p_i}{p_j}=\gamma(\mathbf{u}_i-\mathbf{u}_j)^\top\boldsymbol{\nu}.
    \label{eq:appendix-geometry-odds}
\end{equation}
If $\mathbf{u}_i\ne \mathbf{u}_j$, dividing this quantity by $\gamma\|\mathbf{u}_i-\mathbf{u}_j\|$ gives the signed distance from $\boldsymbol{\nu}$ to the pairwise decision hyperplane in the normalized feature coordinates.
Furthermore, the historical-versus-current margin is
\begin{equation}
    m=\max_{o\in O}s_o-\max_{n\in N}s_n
     =\log\frac{\max_{o\in O}p_o}{\max_{n\in N}p_n}.
    \label{eq:margin-probability-identity}
\end{equation}
\end{theorem}
\noindent\textit{Proof.}
The common softmax denominator cancels in $p_i/p_j$.
Taking logarithms and substituting the score definition yields Equation~\eqref{eq:appendix-geometry-odds}.
The decision hyperplane is $(\mathbf{u}_i-\mathbf{u}_j)^\top\boldsymbol{\nu}=0$; normalization by its normal's length gives the distance.
Since exponentiation preserves order, the same cancellation applies to the maxima in Equation~\eqref{eq:margin-probability-identity}.\hfill\rule{0.7ex}{0.7ex}

\paragraph{How competition contributes to historical errors.}
Consider any fixed distribution of labeled historical samples $(x,y)$ with $y\in O$ and assume no score ties.
Let $\hat y_O$ and $\hat y_S$ be the predictions on $O$ and $S$, and let $A_O,A_S$ be their accuracies for the same parameters.
Then
\begin{equation}
    A_O-A_S
    =\Pr\!\left\{\hat y_O=y,\ \max_{n\in N}s_n>s_y\right\}.
    \label{eq:competition-error-event}
\end{equation}
Indeed, a sample correct under $S$ is also correct under $O$.
The remaining samples correct under $O$ become incorrect precisely when a current class outranks $y$.
Thus the competition gap in Equation~\eqref{eq:accuracy-decomposition} is the frequency of this event, and adding and subtracting the post-update $A_O$ gives that decomposition exactly.
For a sample correct at the reference, preserved historical ranking together with $m<0$ after updating is a correct-to-incorrect change caused by the current group winning.
Unconditional group crossings also include previously incorrect samples.
On fixed candidates, the net accuracy drop equals the correct-to-incorrect rate minus the incorrect-to-correct rate; neither rate alone is the full-sequence signed AF.
Candidate expansion can already introduce errors at task entry, which a subsequent drift constraint need not repair.

\paragraph{Which competition is retained by the hierarchy.}
The class-arrival partition is $\Pi_t=\{\{o\}:o\in O\}\cup\{N\}$.
It organizes the existing semantic candidates into the competitive relationships retained by the response: each historical class has a separate coordinate, while current arrivals share one.
At the next task, current classes join the historical set and receive separate coordinates; subsequent arrivals form the aggregate group.
The hierarchy concerns these response coordinates; the semantic embedding dimension remains fixed.
The full probability vector determines all pairwise score margins and recovers scores up to a common additive offset.
The hierarchical response intentionally retains less information.
With $r_N(n)=p_n/q_N$ and $\operatorname{LSE}(\mathbf{s}_G)=\log\sum_{c\in G}e^{s_c}$,
\begin{align}
    \log\frac{q_O}{q_N}
      &=\operatorname{LSE}(\mathbf{s}_O)-\operatorname{LSE}(\mathbf{s}_N),\\
    m&=\log\frac{q_O}{q_N}+\log\max_o r_O(o)-\log\max_n r_N(n).
    \label{eq:group-margin-decomposition}
\end{align}
Hence $\mathbf{r}=((p_o)_{o\in O},q_N)$ specifies historical conditional discrimination and aggregate group competition, while leaving the current-group conditional distribution free.
These are identities for the model's scores; they neither require calibrated posteriors nor recover every aspect of feature geometry.

\paragraph{Historical-node margins in Figure~\ref{fig:challenges}.}
\label{app:intro-geometry-probe}
The figure evaluates the same 200 Photo test nodes (100 Film Photography and 100 Flashes) before and after learning Sports-Fitness at task 5 of the six-task setting, using the SGD readout.
Both checkpoints score the same 35 historical classes and 13 Sports-Fitness classes.
For true class $y$, the horizontal coordinate is $s_y-\max_{o\in\mathcal O_t\setminus\{y\}}s_o$ and the vertical coordinate is $s_y-\max_{n\in\mathcal N_t}s_n$.
These are score margins computed directly from the logits, rather than a projection of node embeddings.
The zero lines mark ties with the true class; the diagonal in the lower-left quadrant separates historical-class and Sports-Fitness errors when both beat the true class.
Colors show the winning class group, and shapes identify the two Photo classes.
After adaptation, 26 of these nodes retain a positive historical margin but are predicted as Sports-Fitness.

\paragraph{Continuous boundary trajectories.}
\label{app:boundary-trajectories}
Figure~\ref{fig:motivation-radial-stalling} and Figure~\ref{fig:case-boundary-learning}c,d use the same fixed 16-node training minibatch from CiteSeer, Task 2 of the four-task setting.
Starting from the saved task-1 readout, we replay the first 120 task-2 updates in their original data order to obtain a state inside the protection boundary.
The two update rules independently optimize the fixed minibatch from this state, sharing the support, response definition, and budgets.
Their first 17 updates produce identical parameters, losses, and update norms.
Relative step 0 is the next update, whose initial task KL is $0.004827$, above 95\% of its $0.005$ budget.
The local step-radius budget is $0.002$ through step 0, then decreases linearly to $0.0002$ over 32 further updates; the task-level KL budget remains fixed.
All accepted updates satisfy both nonlinear KL checks. At the final point, training loss is $1.26672$ for radial shrinking and $1.06853$ for tangent correction.
These curves measure optimization on the fixed training minibatch; task-level test results are reported separately below.

\paragraph{Single-step boundary diagnostic.}
\label{app:boundary-single-step}
We additionally probe a saved task-2 hierarchical/radial checkpoint from the four-task setting to compare individual feasible updates.
The task KL is at its $0.005$ limit. Both update rules use the same first 16-node training minibatch, support, gradient, and budgets.
The raw proposal violates the task limit ($0.005062$); radial scaling yields an L2 step of $1.29\!\times\!10^{-7}$, whereas the feasible tangent step has length $0.00713$ and lowers the batch cross-entropy from $1.782$ to $1.741$.
The corrected step has local KL $1.39\!\times\!10^{-6}<2\!\times\!10^{-6}$.
Tangent correction reserves inward slack and is followed by nonlinear feasibility checks, so the accepted step need not lie exactly along the first-order tangent.
This one-step diagnostic measures update feasibility and batch loss; the full-task comparison is reported separately in Section~\ref{sec:case-study}.

\subsection{Hierarchical response and conditional prediction stability}
\label{app:response-details}

Fix a task with nonempty historical and current class sets, and suppress the input and task indices.
All probabilities below are strictly positive for finite logits.
Expanding the KL of the hierarchical response gives
\begin{align}
    D_H
    &=\sum_{o\in\mathcal O_t}p_o^a\log\frac{p_o^a}{p_o}
      +q_N^a\log\frac{q_N^a}{q_N} \notag\\
    &=\sum_{o\in\mathcal O_t}q_O^a r_O^a(o)
      \left(\log\frac{q_O^a}{q_O}+\log\frac{r_O^a(o)}{r_O(o)}\right)
      +q_N^a\log\frac{q_N^a}{q_N} \notag\\
    &=D_{\mathrm{KL}}\!\left((q_O^a,q_N^a)\,\|\,(q_O,q_N)\right)
      +q_O^a D_{\mathrm{KL}}(\mathbf{r}_O^a\|\mathbf{r}_O).
\end{align}

\paragraph{Theorem~\ref{prop:response-stability}: statement and proof.}
For a reference class $c\in O$, define
\begin{equation}
    \delta_R=p_c^a-\max\!\left\{\max_{o\in O\setminus\{c\}}p_o^a,\ q_N^a\right\}>0,
    \label{eq:response-margin}
\end{equation}
where the inner maximum is zero if $O=\{c\}$.
If $D_H<\delta_R^2/2$, the updated unified prediction remains $c$, regardless of redistribution within $N$.
When $c$ is the true label, this excludes a correct-to-incorrect change at this input.
\noindent\textit{Proof.}
Pinsker's inequality gives $\eta=\operatorname{TV}(\mathbf{r}^a,\mathbf{r})\leq\sqrt{D_H/2}$, so every response coordinate changes by at most $\eta$.
For each other historical class $o$, $p_c-p_o\geq\delta_R-2\eta>0$.
Likewise, $p_c-q_N\geq\delta_R-2\eta>0$.
Since $p_n\leq q_N$ for each $n\in N$, class $c$ exceeds every candidate.\hfill\rule{0.7ex}{0.7ex}

This certificate is sufficient, not necessary: reference correctness alone does not imply $\delta_R>0$, because the certificate compares one historical class against the entire current-group mass.
At the response level, probabilities over historical A/B and current C/D can change from $(0.40,0.10,0.25,0.25)$ to $(0.40,0.10,0.49,0.01)$.
The hierarchical response stays $(0.40,0.10,0.50)$ and $D_H=0$, yet the winner changes from A to C.
Thus preserving aggregate competition alone cannot guarantee a fixed strongest-competitor margin.

\paragraph{Separate conditional and group guarantees.}
In particular, $D_{\mathrm{KL}}(\mathbf{r}_O^a\|\mathbf{r}_O)\leq D_H/q_O^a$.
Suppose there are at least two historical classes and the reference conditional distribution has a unique maximizer $c_1$, with gap
\begin{equation}
    \Delta_O=r_O^a(c_1)-\max_{o\neq c_1}r_O^a(o)>0.
\end{equation}
If $D_H<q_O^a\Delta_O^2/2$, then the highest-scoring historical class remains $c_1$.
To see this, Pinsker's inequality gives
$\operatorname{TV}(\mathbf{r}_O^a,\mathbf{r}_O)\leq\sqrt{D_{\mathrm{KL}}(\mathbf{r}_O^a\|\mathbf{r}_O)/2}$.
The probability of any single class changes by at most this total variation.
Hence every pairwise gap between $c_1$ and another historical class remains positive when
$2\operatorname{TV}(\mathbf{r}_O^a,\mathbf{r}_O)<\Delta_O$.
This establishes stability of the historical maximizer, not correctness with respect to the true label or dominance over current classes.

One can also obtain a conservative condition for a historical class to remain globally dominant.
Let $\epsilon_m=D_{\mathrm{KL}}((q_O^a,q_N^a)\|(q_O,q_N))$,
$\epsilon_c=D_{\mathrm{KL}}(\mathbf{r}_O^a\|\mathbf{r}_O)$,
$\eta_m=\sqrt{\epsilon_m/2}$, $\eta_c=\sqrt{\epsilon_c/2}$, and $r_*^a=\max_o r_O^a(o)$.
Writing $[z]_+=\max(z,0)$, the condition
\begin{equation}
    [q_O^a-\eta_m]_+[r_*^a-\eta_c]_+>q_N^a+\eta_m
\end{equation}
ensures that a historical-class probability exceeds every current-class probability.
The left-hand side lower-bounds the probability of the reference best historical class; the right-hand side upper-bounds total current mass and therefore its largest individual probability.
Combining this condition with conditional ranking stability preserves the reference best historical class under unified prediction.
The condition is sufficient and can be loose because it bounds the strongest current competitor by the entire group mass.

These statements are pointwise.
For the empirical source-weighted distribution $w$ used in training, an average bound $\mathbb E_w[D_H]\leq\epsilon$ implies only
\begin{equation}
    \Pr_{x\sim w}\{D_H(x)\geq\tau\}\leq\epsilon/\tau,
    \qquad \tau>0,
\end{equation}
by Markov's inequality.
It does not provide a test-distribution guarantee or uniform conditional protection when $q_O^a$ is small.

\subsection{Implicit Fisher products}
\label{app:fisher-products}

All quantities in this subsection are evaluated at the current iterate $\boldsymbol{\theta}_k$.
For support input $x$, let $\mathbf{z}(x)=(s_c(x))_{c\in\mathcal C_{\leq t}}$,
$\mathbf{y}=\mathbf{W}_k\widetilde{\mathbf{h}}$, $\boldsymbol{\nu}=\mathbf{y}/\|\mathbf{y}\|_2$, and let $\mathbf{U}$ have normalized class embeddings as rows.
For a parameter direction $\mathbf{v}=\operatorname{vec}(\mathbf{V})$, the score Jacobian product is
\begin{equation}
    \dot{\mathbf{z}}=\mathbf{J}_x\mathbf{v}
    =\gamma \mathbf{U}\frac{(\mathbf{I}-\boldsymbol{\nu}\boldsymbol{\nu}^\top)\mathbf{V}\widetilde{\mathbf{h}}}{\|\mathbf{y}\|_2}.
\end{equation}
The implementation clamps the norm used in these derivative calculations below by $10^{-8}$.
For nonempty historical and current groups, write $\mathbf{r}_O$ and $\mathbf{r}_N$ for their respective conditional softmax distributions and set
$\mathbf{b}=[\mathbf{r}_O;-\mathbf{r}_N]$, with historical classes listed first.
The response Fisher in logit coordinates is
\begin{equation}
    \mathbf{H}_x=q_Oq_N\mathbf{b}\mathbf{b}^\top+
    \begin{bmatrix}
        q_O\big(\operatorname{Diag}(\mathbf{r}_O)-\mathbf{r}_O\mathbf{r}_O^\top\big)&0\\
        0&0
    \end{bmatrix}.
\end{equation}
The first term measures local curvature of aggregate group competition; the second is the historical conditional Fisher weighted by $q_O$.
Consequently,
\begin{equation}
    \mathbf{G}_k\mathbf{v}=\frac{1}{K}\sum_{r=1}^{K}\frac{1}{|\mathcal S_r|}
    \sum_{x\in\mathcal S_r}\mathbf{J}_x^\top \mathbf{H}_x(\mathbf{J}_x\mathbf{v})+\lambda \mathbf{v}.
\end{equation}
The conditional curvature product is computed as
$\mathbf{r}_O\odot\big(\dot{\mathbf{z}}_O-(\mathbf{r}_O^\top\dot{\mathbf{z}}_O)\mathbf{1}\big)$,
without materializing its class-by-class matrix.
Likewise, Jacobian products and their transposes are evaluated through the normalized readout.
The source weights are $1/(K|\mathcal S_r|)$ per row.
No dense parameter Fisher or Hessian is stored.
When the historical group is empty, the response is constant and only damping remains.

\subsection{Support construction}
\label{app:support-details}

The initial support contains external unlabeled reference features, with at most 1024 retained rows.
At the end of task $t$, a seeded random permutation of the current training features selects at most 512 rows without replacement.
These form an additional historical source.
Once any historical source exists, the external reference retains at most 512 rows.
The remaining portion of the 1024-row budget is apportioned among historical sources in proportion to their currently retained row counts, using integer quotas and deterministic remainder allocation.
When capacity permits, every historical source retains at least one row.
Within a source, deterministic evenly spaced indices select the retained rows.
Repeated compression can progressively reduce early-task coverage; the memory is not a balanced reservoir over all historical training samples.

Stored source metadata identifies the task, source, and candidate classes, but individual historical node labels are not used for supervision.
Although checkpoints also retain response-related state, each task recomputes the constraints on surviving feature rows with its expanded candidates and task anchor.
They do not compare distributions with mismatched candidate sets.
The feature-row budget excludes class embeddings, reference responses, readout anchors, and temporary optimization buffers.
Feature storage is $O(M(d+1))$, and the readout contains $m(d+1)$ parameters.

\subsection{Proposal correction and feasibility search}
\label{app:solver-details}

\begin{hirpalgorithm}{HiRP training and inference}
\label{alg:hirp}
\textbf{Inputs:} frozen node features $\mathbf{h}$, class-text encoder, pretrained readout $\mathbf{W}$; current labeled task data; external unlabeled support; $M=1024$, $\delta=0.002$, $\rho=0.1$, $\lambda=0.01$.\par
\textbf{Initialize:} retain at most $M$ external feature rows; set $L=\delta^2/2$, $B=\rho^2/2$.\par
\begin{enumerate}
\setlength{\itemsep}{2pt}\setlength{\parsep}{0pt}\setlength{\topsep}{3pt}
\item \textbf{For each arriving task $t$, repeat steps 2--8:}
  encode newly arrived class descriptions; form $\mathcal O_t$, $\mathcal N_t$, and $\mathcal C_{\leq t}$.
\item Freeze the task support $\mathcal M_t$ and set $\boldsymbol{\theta}_0=\operatorname{vec}(\mathbf{W})$.
  Evaluate $\mathbf{r}_t(x;\boldsymbol{\theta}_0)$ on the expanded candidates for every support row.
\item \textbf{For each current-data minibatch in the three training epochs, repeat steps 4--7:}
  compute the label-smoothed current cross-entropy gradient $\mathbf{g}_k$; set
  $\mathbf{G}_k=\mathbf{F}_k+\lambda \mathbf{I}$ using the current-iterate response.
\item Check that $\widehat D_t(\boldsymbol{\theta}_0,\boldsymbol{\theta}_k)\leq(1+10^{-5})B$; otherwise raise an error.
  Compute $\mathbf{v}_k\approx \mathbf{G}_k^{-1}\mathbf{g}_k$ by CG (at most 12 steps); form $\mathbf{d}_k$ using Equation~\eqref{eq:natural-proposal}, flooring $\mathbf{g}_k^\top \mathbf{v}_k$ at $10^{-20}$.
  A zero gradient produces a zero proposal.
\item If $\widehat D_t(\boldsymbol{\theta}_0,\boldsymbol{\theta}_k)\geq0.95B$, compute $\mathbf{a}_k$ and $\mathbf{u}_k\approx \mathbf{G}_k^{-1}\mathbf{a}_k$.
  If $\mathbf{a}_k^\top \mathbf{d}_k>0$, $\mathbf{a}_k^\top \mathbf{u}_k>10^{-20}$, and the projected norm exceeds $10^{-12}$, form Equations~\eqref{eq:tangent-proposal} and~\eqref{eq:inward-proposal} and rescale into the local metric ball.
  Replace $\mathbf{d}_k$ with this correction only when its inner product with $\mathbf{g}_k$ is negative.
\item Set $\boldsymbol{\theta}'=\boldsymbol{\theta}_k+\mathbf{d}_k$.
  If local KL exceeds $L$, shrink the displacement by 48 interval-search steps retaining an evaluated feasible endpoint.
  If task KL then exceeds $B$, repeat this search on the remaining displacement.
\item Recompute both nonlinear KL values.
  Commit $\boldsymbol{\theta}_{k+1}=\boldsymbol{\theta}'$ only if finite and both pass relative tolerance $10^{-4}$; otherwise raise an error without committing.
\item \textbf{At task end:} set $\mathbf{W}$ to the final iterate reshaped as a readout matrix; sample up to 512 current-training features without their supervision labels, append their source metadata, and compress to $M$ rows by Appendix~\ref{app:support-details}.
\end{enumerate}
\textbf{Inference:} use the shared readout and all arrived class descriptions in Equation~\eqref{eq:global-prediction}; no task routing or score calibration.
\end{hirpalgorithm}

Conjugate gradients start from zero and stop after 12 iterations or when the residual norm is at most $10^{-5}$ times its initial value.
The proposal denominator uses $\mathbf{g}_k^\top \mathbf{v}_k$ with a numerical floor of $10^{-20}$.
No historical supervised loss is included in $\mathbf{g}_k$; the current minibatch cross-entropy uses label smoothing of $0.2$.

Set $B=\rho^2/2$ and $L=\delta^2/2$.
If the current task divergence is at least $0.95B$, compute its gradient $\mathbf{a}_k$ and the CG approximation $\mathbf{u}_k$ to $\mathbf{G}_k^{-1}\mathbf{a}_k$.
The tangent correction is attempted only when $\mathbf{a}_k^\top \mathbf{d}_k>0$ and $\mathbf{a}_k^\top \mathbf{u}_k>10^{-20}$.
After projection, a metric norm exceeding $10^{-12}$ is required before normalization.
We normalize the tangent proposal and introduce inward slack using
\begin{equation}
    \bar{\mathbf{d}}_k=\delta\frac{\mathbf{d}_k^{\mathrm{tan}}}{\|\mathbf{d}_k^{\mathrm{tan}}\|_{\mathbf{G}_k}},
    \qquad
    \mathbf{d}_k^{\mathrm{in}}=\bar{\mathbf{d}}_k
    -\frac{\|\bar{\mathbf{d}}_k\|_{\mathbf{G}_k}^2}{\mathbf{a}_k^\top \mathbf{u}_k}\mathbf{u}_k.
    \label{eq:inward-proposal}
\end{equation}
The inward adjustment in Equation~\eqref{eq:inward-proposal} is then rescaled by
$\min(1,\delta/\|\mathbf{d}_k^{\mathrm{in}}\|_{\mathbf{G}_k})$.
It replaces the original proposal only if its inner product with $\mathbf{g}_k$ is negative.
With an exact inverse, the tangent term removes the linearized outward component, and the inward adjustment contributes $-\|\bar{\mathbf{d}}_k\|_{\mathbf{G}_k}^2$ to the task-divergence directional derivative.
This supplies first-order inward slack, rather than an analytic bound on all higher-order terms.

The solver first evaluates the local KL of the proposed update.
If it exceeds $L$, it searches a scalar multiplier in $[0,1]$ for 48 bisection iterations, retaining an evaluated feasible lower endpoint.
It next checks the task KL and, if necessary, applies the same procedure to the remaining displacement.
Both divergences are recomputed for the final proposal.
KL evaluations use double-precision log-probabilities; accepted updates must pass both budget checks within the implementation's relative tolerance of $10^{-4}$.
The final verification matters because the nonlinear KL level sets need not be convex along an arbitrary parameter ray.
The interval search is a feasibility procedure and does not establish a maximal feasible step or a general convergence guarantee.
If the final proposal is nonfinite or fails a budget check, the implementation raises an error before committing it.
It also rejects a current state already outside the task budget beyond its tolerance.
A zero gradient gives a zero CG proposal through the denominator floor; no update is required in that case.
The two-reference KL constraints and exact-inverse tangent projection specify the mathematical construction; the $95\%$ trigger, inward adjustment, finite CG solve, and interval searches specify its numerical realization.

\subsection{First-task optimization and attribution}
\label{app:first-stage-attribution}

At the first task, the response is constant and $\mathbf{F}_k=0$.
For a nonzero gradient and an exact inverse, Equation~\eqref{eq:natural-proposal} reduces to
\begin{equation}
    \mathbf{d}_k=-\frac{\delta}{\sqrt{\lambda}}\frac{\mathbf{g}_k}{\|\mathbf{g}_k\|_2},
    \qquad \|\mathbf{d}_k\|_2=0.02.
    \label{eq:first-stage-update}
\end{equation}
This normalized step is the first-task specialization of the same update rule, before historical response constraints become active.
Mean first-task accuracies are $81.49\%$ for HiRP versus $78.84\%$ for SGD in the four- and six-task settings, and $83.77\%$ versus $79.68\%$ in the ten-task setting; the four- and six-task settings share the same first task.

The final FA/AA comparisons evaluate the complete algorithms, including this optimization difference.
The 32-step intervention in Section~\ref{sec:case-study} starts both methods from the same later HiRP checkpoint and matches parameter-step norms, isolating direction effects at those selected states.
The response/update factorial ablation holds the data, initialization, support rules, and numerical budgets fixed across its variants.
Together these analyses distinguish full-system accuracy, direction effects at common states, and conditional response/solver effects.

\paragraph{Full-sequence controls from a common first-task checkpoint.}
\label{app:common-start-full-sequence}
In the ten-task setting, all methods start task 2 from the same HiRP task-1 checkpoint, with task-1 accuracy $83.77\%$, and complete tasks 2--10.
They share frozen graph and text encoders, the trainable readout, training batches, three epochs per task, batch size 16, and label smoothing $0.2$.
The Fisher-proposal control retains the hierarchical response, feature support, damped Fisher metric, and normalized proposal rule, while removing both KL feasibility checks, radial shrinking, and tangent correction.
SGD uses its standard learning rate of $5\times10^{-4}$.

\begin{table}[!htbp]
    \centering
    \caption{Full-sequence comparisons in the ten-task setting from a shared task-1 checkpoint. Zero-shot mean is the unweighted accuracy over History, Computer, and Instagram after task 10. Values are percentages; higher is better.}
    \label{tab:common-start-transfer}
    \begingroup
    \small
    \setlength{\tabcolsep}{8pt}
    \renewcommand{\arraystretch}{1.12}
    \begin{tabular}{@{}lrr@{}}
        \toprule
        Method & AA $\uparrow$ & Zero-shot mean $\uparrow$ \\
        \midrule
        SGD & 50.71 & 56.82 \\
        Fisher proposal without KL constraints & 56.84 & 55.20 \\
        HiRP (ours) & \textbf{69.74} & \textbf{58.27} \\
        \bottomrule
    \end{tabular}
    \endgroup
\end{table}

HiRP reaches $69.74\%$ AA versus $56.84\%$ for the Fisher-proposal control, a $12.91$-point gain (Table~\ref{tab:common-start-transfer}).
Restricting the final average to tasks 2--10 yields $69.11\%$ versus $58.04\%$.
For SGD, replacing its own task-1 checkpoint with HiRP's increases final AA from $49.83\%$ to $50.71\%$, only $0.88$ points.
These results support an additional benefit from response constraints and their update corrections beyond the first-task starting point and the Fisher proposal alone.

\paragraph{Continual accuracy and zero-shot transfer with matched frozen encoders.}
We evaluate the same final checkpoints on History, Computer, and Instagram using each domain's complete class set, without target-domain adaptation.
These domains contribute neither continual-training examples nor feature support and are not used for model selection; the reported zero-shot mean weights the three domains equally.
Even with the same frozen encoders, trainable readout, and task-1 starting point, SGD is lower than HiRP in both AA ($50.71\%$ versus $69.74\%$) and mean zero-shot accuracy ($56.82\%$ versus $58.27\%$).
Thus, in this matched comparison, HiRP achieves a better balance between learning arrived classes and retaining transfer to held-out graph domains.

\paragraph{Common-start, norm-matched intervention.}
We intervene at task 3 of the four-task setting and task 5 of the six- and ten-task settings.
For each run, HiRP and SGD start from the same HiRP checkpoint and receive the same 32 current-training batches.
Each SGD update is rescaled to match the corresponding HiRP update's Frobenius norm.
For fixed candidate sets, define
\begin{equation}
    m_t(x;\mathbf{W})=\max_{o\in\mathcal O_t}s_o(x;\mathbf{W})
    -\max_{n\in\mathcal N_t}s_n(x;\mathbf{W}).
    \label{eq:old-new-margin}
\end{equation}
The crossing rate is the fraction of all historical test samples with $m_t>0$ before and $m_t<0$ after the intervention.
It measures a change in the winning group, not necessarily a correct-to-incorrect prediction.
HiRP reduces this rate in all three streams; in the ten-task setting it falls from $2.627\%$ to $0.254\%$, while current-class global accuracy changes from $91.86\%$ to $91.71\%$.
Historical-sample global accuracy improves by $0.56$, $0.45$, and $2.94$ points in the four-, six-, and ten-task settings, respectively; current-sample accuracy decreases by $0.05$, $0.18$, and $0.15$ points.
Matching update norms rules out smaller per-step parameter displacement as the sole explanation at these common starting states.

\paragraph{Case-study diagnostics.}
\label{app:case-study}
Figure~\ref{fig:case-preservation}a,b uses all 1180 historical test nodes at task 5 of the ten-task setting and the same 32 batches from 129 current training nodes.
The mass-only control shares the starting readout, support, and numerical KL budgets with HiRP; SGD* matches HiRP's step norm at each iteration.
Mean old-only accuracy changes by $-2.60$, $-0.62$, and $+0.11$ points for SGD*, Mass, and HiRP, respectively.
Their old-to-new crossing rates are $2.63\%$, $0.54\%$, and $0.25\%$; new-class accuracies are $91.86\%$, $91.76\%$, and $91.71\%$.
Tangent correction is inactive during these short interventions, so panels a,b examine the response and its induced update rather than activation of the tangent correction.

Additional single-step and complete-task diagnostics use CiteSeer, Task 2 of the four-task setting.
The single-step experiment uses the boundary checkpoint described in Appendix~\ref{app:boundary-single-step}; retained length is the accepted norm divided by the raw proposal norm, and loss decrease is relative to the initial batch loss.
The raw proposal norm is $0.0076947$, radial norm is $1.2910\!\times\!10^{-7}$, and tangent norm is $0.0071312$.
The initial loss $1.78164$ becomes $1.78164$ under radial shrinking and $1.74053$ under tangent correction.
Separately, complete task-2 runs of hierarchical radial and tangent updates yield new-class test accuracies of $64.47\pm0.90\%$ and $66.67\pm0.80\%$, and old-class accuracies of $81.18\pm1.76\%$ and $80.87\pm1.38\%$.
The new-class accuracy comparison uses complete-task training.

\subsection{What the four-cell ablation controls}
\label{app:ablation-controls}

Table~\ref{tab:ablation} is a $2\times2$ factorial comparison.
The response factor selects aggregate group competition alone or additionally preserves historical conditional discrimination; the solver factor selects scalar shrinking or tangent correction followed by feasibility checks.
Within each paired comparison, all four variants share the pretrained readout, frozen feature cache, class order, current-data batches, support-selection rule, local/task radii, damping, and training epochs.
Their first-task readout tensors are identical across variants in every paired run.
Paired protocol and first-task equality records are supplied in \texttt{reproducibility/ablation\_controls.json}.

Holding the solver fixed compares the two response designs; holding the response fixed compares the two solvers.
For example, with tangent updates fixed, the hierarchical response improves FA over the binary response by $1.59$, $1.22$, and $13.12$ points on whole4, whole6, and global10.
With the hierarchical response fixed, tangent correction improves FA over radial shrinking by $4.57$, $4.21$, and $1.07$ points.
The gains interact: on global10, tangent minus radial is $-2.18$ points with binary mass but $+1.07$ with the hierarchical response.
The interaction makes these conditional contrasts nonadditive.

The response change also changes the Fisher geometry and accepted displacement, so it measures the complete response-induced update under equal numerical budgets.

\section{Additional Experimental Details and Results}
\label{app:experimental-details}

\subsection{Data and protocols}

Table~\ref{tab:data-splits} specifies the supplied splits used for adaptation and evaluation.
Cora denotes the seven-class citation graph with 2,708 nodes; it is not Cora-Full.
Photo denotes the twelve-class text-attributed product graph used in our GraphCLIP pipeline.
Classes from different graphs receive disjoint identifiers.
The full training split is used within each task's class subset, and the train and test sample pools are disjoint.
The configuration names \texttt{whole4}, \texttt{whole6}, and \texttt{global10} denote the 4-dataset/4-task, 6-dataset/6-task, and 4-dataset/10-task settings, respectively. Result files store FA under the key \texttt{Global}.
The exact class-ID maps and per-task cumulative candidates are included in \texttt{reproducibility/protocols/}; \texttt{reproducibility/splits/} contains the four original-domain node-ID partitions.
Global IDs $0$--$6$, $7$--$12$, $13$--$22$, and $23$--$34$ refer respectively to Cora, CiteSeer, WikiCS, and Photo; whole6 appends Sports-Fitness ($35$--$47$) and Actor ($48$--$52$).
The ten tasks of global10 contain $4,3,3,3,4,3,3,4,4,4$ classes in that order; the supplied TSV maps these ranges to the original class identities.

\begin{table}[t]
\centering
\caption{Exact dataset instances and prescribed splits. The single-graph schedule lists the classes introduced at each task; the two external graphs are used in the six-task setting. WikiCS retains its official partially assigned split.}
\label{tab:data-splits}
\small
\setlength{\tabcolsep}{3pt}
\begin{tabular}{lrrrrrl}
\toprule
Dataset & Nodes & Classes & Train & Valid. & Test & Single schedule \\
\midrule
Cora & 2,708 & 7 & 1,624 & 542 & 542 & $2,2,3$ \\
CiteSeer & 3,186 & 6 & 1,911 & 637 & 638 & $2,2,2$ \\
WikiCS & 11,701 & 10 & 580 & 1,769 & 5,847 & $2,2,2,2,2$ \\
Ele-Photo & 48,362 & 12 & 29,017 & 9,672 & 9,673 & $2,2,2,2,2,2$ \\
Sports-Fitness & 173,055 & 13 & 34,611 & 17,305 & 121,139 & -- \\
Actor (HetGB) & 4,416 & 5 & 2,116 & 1,411 & 889 & -- \\
\bottomrule
\end{tabular}
\end{table}

Cora and Photo use fixed NumPy default-generator seed-0 partitions; CiteSeer uses the saved NumPy RandomState seed-0 partition.
WikiCS uses official mask column 0, with 3,505 nodes outside the supervised train/validation/test assignment.
Sports-Fitness uses the supplied aligned \texttt{Fitness.csv}/\texttt{Fitness.pt} graph and a fixed 20/10/70 split generated with PyTorch seed 88.
Actor uses the masks in the supplied HetGB \texttt{Actor.npz} archive.
The source files and original-node split identities for both external graphs are described in \texttt{external\_dataset\_identity.json}; compact cache indices are not substituted for original node identities.

\paragraph{Graph context and label access.}
These are fixed, transductive graph instances with sequential label access.
For the four original domains, the GraphCLIP pipeline uses cached root-node graphs derived from the supplied \texttt{target\_data} edge lists, augmented with 32-dimensional random-walk positional encodings.
For the two external domains, root-node graphs use walks of 256 steps (Sports-Fitness) or 64 steps (Actor), restart probability 0.8, and the same positional-encoding dimension.
Unlabeled neighbors may belong to other tasks or splits: the graph context is not restricted to currently labeled nodes.
The supervised loss uses only current-task training roots, and evaluation uses cumulative candidates for every test root.
Historical training roots enter the bounded support only after their task; validation and test roots are not added as support rows.
This setting therefore does not claim an inductive protocol in which future nodes, edges, or attributes are unavailable.

\paragraph{Frozen model and class semantics.}
The checkpoint is \texttt{pretrained\_graphclip.pt}; its architecture and readout settings are recorded in \texttt{model\_and\_reference.json}.
Its graph encoder has 12 GPS layers, hidden size 1024, and 8 attention heads; 384-dimensional node inputs and 32-dimensional positional encodings yield a 2048-dimensional feature by concatenating mean and center-node representations.
The trainable affine readout is $384\times2049$, including its bias, with 786,816 parameters.
The text encoder/tokenizer is \texttt{sentence-transformers/all-MiniLM-L6-v2}, with attention-mask mean pooling, maximum class-text length 512, and unit-normalized 384-dimensional class embeddings.
We retain the checkpoint's fixed cosine-logit scale $\gamma=14.7611256$.
Encoders run in evaluation mode; the frozen feature cache is shared across methods that adapt the readout.

Class inputs concatenate the dataset template, class name, and supplied description.
The templates are ``this paper has a topic on \{class\}'' (Cora), ``good paper of \{class\}'' (CiteSeer), ``it belongs to \{class\} research area'' (WikiCS), ``this product belongs to \{class\}'' (Photo/Sports-Fitness), and ``this actor belongs to \{class\}'' (Actor).
The complete strings, including whitespace and descriptions, are included in \texttt{class\_texts.json}.
These semantic prototypes are frozen; the active classifier at task $t$ uses only the arrived candidate IDs.
Precomputing the immutable text bank does not expose future-class labels or future learned model parameters to the task's loss.

\paragraph{External reference and feature memory.}
The initial external pool contains 512 root graphs: 256 from arXiv-2023 and 256 from PubMed, selected without replacement using seeds 20260721 and 20260722.
It is the recorded \texttt{graph\_zscl\_reference\_pool\_512\_rwpe32.pt} asset, not the separate five-source balanced pool used by other exploratory runs.
The source datasets are distinct from the six target datasets; this is dataset/ID separation, not a claim of exhaustive semantic deduplication across corpora.
The reference may share source domains with GraphCLIP pretraining.
Its graph IDs, source counts, and feature-memory settings are included in \texttt{model\_and\_reference.json}.
Actual initial support is 512 rows, with the 1024-row cap applied after historical training features are added.
One fully occupied float32 feature support uses 8,392,704 bytes (about 8.00 MiB); the readout uses 3,147,264 bytes (about 3.00 MiB).
Reference responses, anchors, class embeddings, and solver workspaces are additional state and are not included in those two counts.

\subsection{Baseline implementations}
\label{app:baseline-settings}

\paragraph{GraphCLIP joint fine-tuning.}
GraphCLIP starts from the same pretrained checkpoint as HiRP and jointly updates the graph encoder, shared readout, text encoder, and logit scale.
We use the bidirectional graph--text pairwise cross-entropy implementation released with ADAligner~\citep{liu2026adaligner} in its public downstream code.\footnote{\url{https://github.com/karmaisacat-13/ADAligner}, commit \texttt{5e4c7a3}, \texttt{transfer/eval\_prompt.py}.}
Optimization uses AdamW, learning rate $10^{-5}$, weight decay $10^{-5}$, batch size 4, five epochs per task, and gradient clipping at 1.0.
Each task uses all prescribed current-task training examples and continues from the previous task's final model, with a fresh optimizer.
The final epoch supplies the checkpoint; the method uses no LoRA, learned input prompt, replay, or distillation.

\paragraph{Held-out History evaluation.}
Figures~\ref{fig:zero_shot_degradation} and~\ref{fig:zero-shot-retention} use the ten-task stream.
History contributes neither training examples nor feature support and is not used for model selection.
Before continual training ($t=0$) and after each task, evaluation uses the same 8,311 History test nodes and its complete class set without target-domain adaptation.
GraphCLIP recomputes class text embeddings using its updated text encoder, while HiRP uses its frozen text encoder and current readout.

\paragraph{Readout and prototype baselines.}

SGD updates the shared GraphCLIP readout using current-task cross-entropy.
LwF adds distillation of historical-class responses on current inputs, with weight 1 and temperature 2.
Both use a learning rate of $5\times10^{-4}$, three epochs, batch size 16, and label smoothing 0.2.
Neither has a historical supervised replay loss.
HiRP uses the same current-data training budget, replacing the ordinary readout update with its response-constrained rule.
EWC, MAS~\citep{aljundi2018mas}, and OGD~\citep{farajtabar2020ogd} are parameter-regularization or projection baselines on the shared readout.
For whole4 and global10, their saved checkpoints are reevaluated on the common fixed test split; whole6 is evaluated during its full training sequence.
All use the same final metric definitions, while this checkpoint/evaluation provenance is retained in the reproducibility records.
Cosine and TEEN use their own graph representations and prototype rules.

The LoRA comparison adapts GraphCLIP using rank-78 adapters in its final graph block and a shared readout, with 2,064,768 trainable parameters in total.
It uses three epochs per task, batch size 32, learning rate $5\times10^{-4}$, and label smoothing 0.2.
G2LoRA is a separate GraphCLIP adaptation of the published mechanisms, not a claim to reproduce the original paper's complete system or few-shot results.
It uses rank-32 graph and text adapters, four epochs per task, batch size 32, learning rate $10^{-5}$, weight decay 0.05, and unsmoothed cross-entropy.
Both use all prescribed current-class training samples and cumulative test candidates.
Consequently, comparisons share the data protocol but do not equate trainable parameter counts or optimization budgets.

The cross-graph and single-graph comparisons additionally use the corresponding graph/text implementations of GCN~\citep{kipf2017gcn}, Cosine, TEEN~\citep{wang2023teen}, BERT~\citep{devlin2019bert}, and the explicitly identified variants of TPP~\citep{niu2024tpp} and SimpleCIL~\citep{zhou2025simplecil}.
GCN is the LLM4GCL~\citep{cheng2025llm4gcl} \texttt{BareGNN} implementation instantiated as a two-layer GCN with hidden size 128 for the cross-graph runs; it trains on node input features and edges without the pretrained GraphCLIP graph encoder.
BERT sequentially adapts its text model using the upstream LoRA configuration (rank 5, $\alpha=16$, dropout 0.05).
Cosine and TEEN retain their graph representations and prototype rules.
SimpleCIL (BERT) adapts BERT with LoRA on the first task, then keeps its representations fixed and updates class prototypes; this differs from the RoBERTa configuration described in G2LoRA.
Cross-graph prototype updates use all prescribed current-class training samples.

TPP-style denotes the LLM4GCL-derived adapter with graph-smoothed task profiles, per-node predicted routing, and independent linear classification heads on input features.
The executed implementation omits the original TPP graph prompts and GNN classification path; its scores therefore characterize this variant, not a complete reproduction of the original method.
Its task heads remain fixed after their respective tasks, but newly available task prototypes can change the routing of historical nodes.
Inference never uses a ground-truth task identifier or ground-truth task grouping of test inputs.
The adapters use the same node indices, class schedules, and cumulative-candidate evaluation as the readout comparisons; validation-based stopping follows each implementation's configuration.
These comparisons retain differing pretraining resources and model capacities.
In the six-task diagnostic, GCN's final accuracy averaged over the first five tasks is only $0.19\%$. The final task accounts for approximately $0.64\%$ of test nodes, contributing to its low sample-weighted FA.

\paragraph{RoBERTa.}
RoBERTa~\citep{liu2019roberta} denotes sequential full fine-tuning of RoBERTa-large and its standard dense--tanh classification head on node text.
The encoder and head are updated at every task using the complete current training split, without LoRA, prototype replacement, or historical supervised replay.
Training cross-entropy and test prediction both use cumulative seen-class candidates; inference does not receive task identities.
We use AdamW with learning rate $2\times10^{-5}$ and weight decay $0.01$, batch size 32, maximum text length 256, 6\% linear warmup followed by linear decay, and gradient clipping at norm 1.
Each task runs three to six epochs, with patience two; current validation accuracy selects the checkpoint, with validation cross-entropy breaking ties.
Parameters and optimizer state are float32, with float16 autocast and gradient scaling; the full selected checkpoint is restored and checked tensor by tensor before evaluation.
On whole6, Sports-Fitness accounts for $87.32\%$ of test nodes: RoBERTa's $75.12\%$ FA therefore coexists with $27.64\%$ AA, underscoring the need to report both metrics.

\paragraph{InfLoRA.}
InfLoRA$^{\dagger}$~\citep{liang2024inflora} implements the interference-reducing subspace mechanism on the final GraphCLIP GPS attention block, with rank-10 K/V adapters and task-wise linear class heads on 2,048-dimensional graph features.
The preceding graph blocks are frozen; the pretrained semantic readout is not used.
Current-task cross-entropy trains the new class rows and active low-rank factors; historical classifier rows stay fixed and inference considers all seen classes.
Activation subspaces and frozen previous factors carry historical information without supervised replay.
Each task uses all prescribed training roots, 20 epochs, batch size 32, Adam with learning rate $5\times10^{-4}$ and zero weight decay, cosine decay, and validation-best checkpoint selection.
This is a last-block GraphCLIP implementation of the mechanism, rather than the original ViT benchmark configuration; its layer scope and optimization budget differ from HiRP.

\paragraph{SD-LoRA.}
SD-LoRA~\citep{wu2025sdlora} implements magnitude--direction decoupling on the final GraphCLIP GPS attention block, using rank-10 Q/V adapters and task-wise linear heads on 2,048-dimensional graph features.
The preceding graph blocks are frozen and the pretrained semantic readout is not used.
Each task learns its new low-rank factors and current classifier rows, retains previous low-rank directions, and updates their projection-specific magnitudes; historical classifier rows remain fixed.
The implementation preserves the effective adapter function when normalizing directions at task boundaries.
Training cross-entropy uses current-class candidates, while inference compares all seen classes without task identities or historical supervised replay.
All runs use 20 epochs per task, batch size 32, SGD with learning rate $10^{-2}$, momentum $0.9$, zero weight decay, and current-validation-best checkpoint selection.
To prevent overflow on Sports-Fitness, all settings use global gradient clipping at norm 1, float64 gradient-norm accumulation, and finite-value checks; no batches are skipped.
Earlier uncorrected attempts are retained separately and excluded.
The table uses the name SD-LoRA; these results describe its recorded last-block GraphCLIP implementation, not the original vision benchmark configuration.

\paragraph{GCN with LLM embeddings.}
GCN\textsubscript{LLMEmb} follows the LLM4GCL feature-enhancement baseline~\citep{cheng2025llm4gcl}.
Frozen 4-bit Llama-3.2-3B provides 3,072-dimensional node features through attention-mask mean pooling, with maximum text length 512.
A two-layer GCN with hidden size 128 and dropout 0.5 is trained sequentially on these features and graph edges.
Each task uses the full current training split, batch size 256, learning rate $10^{-4}$, weight decay $5\times10^{-4}$, and at most 300 epochs with validation-based checkpoint selection.
Inference considers all arrived classes without task identities or historical supervised replay.

\paragraph{GraphGPT adaptation.}
GraphGPT$^{\dagger}$~\citep{tang2024graphgpt} matches multiple graph tokens to source descriptions, then continually tunes a linear projector with the GraphCLIP encoder and Llama-3.2-3B frozen; it uses no language-model LoRA.
The shared alignment checkpoint is trained for three epochs on 22,532 matching groups, with 1,187 held-out groups, from the five recorded GraphCLIP source corpora.
Downstream runs share this checkpoint and vary the training order.
Each task uses full current training data, batch size 16, AdamW with learning rate $2\times10^{-4}$ and zero weight decay, and validation-based selection after three to six epochs.
Greedy generation uses the cumulative class names; unrecognized answers are counted as incorrect.
The saved protocols record prompt construction and evaluation batches; this is an explicit backbone/source-corpus adaptation, not the official GraphGPT pretrained system.

\paragraph{LLaGA.}
LLaGA$^{\dagger}$~\citep{chen2024llaga} uses the neighborhood-detail (ND) architecture with MiniLM384 node inputs, two-hop sampling with ten neighbors per hop, and the upstream Laplacian positional encoding.
A two-layer $495\!\rightarrow\!3072\!\rightarrow\!3072$ projector maps the graph representation to the frozen Llama-3.2-3B input space; no language-model LoRA or GraphCLIP graph encoder is used.
The projector is randomly initialized and continually trained on all current training roots, using batch size 16, AdamW at $2\times10^{-4}$ with zero weight decay, and validation-based selection over three to six epochs.
Neighbor samples are cached per node within each run, rather than resampled per batch.
Greedy decoding uses cumulative class names with unrecognized answers counted as incorrect; final test evaluation uses batch size 32 throughout, while validation retains its recorded per-run batch of 4 or 32.
The records include checked projector save/restore events and the evaluation-batch policy.
These results describe the ND architecture under this continual protocol and language-model choice, not the official pretrained LLaGA system.

\subsection{Comparison conditions and computation}
\label{app:comparison-contract}

Shared splits, candidate masks, and test identities do not imply identical information access, representations, or optimization cost.
Table~\ref{tab:comparison-conditions} records these distinctions.
Cross-graph LLM4GCL adapters use their upstream default configurations with no per-dataset best-setting override; full effective configurations are in \texttt{upstream\_default\_configs.json}.
GCN/Cosine/TEEN use batches of 256 and at most 300 epochs, learning rate $10^{-4}$ and weight decay $5\times10^{-4}$.
TPP-style trains its independent heads for at most 3000 epochs with batch size 200 and the same rate/decay; the unused original prompt/pretraining options in its YAML do not activate those mechanisms.
BERT and SimpleCIL use \texttt{google-bert/bert-base-uncased}, learning rate $2\times10^{-4}$, weight decay 0.05, maximum text length 256, and at most 20 epochs, with batch sizes 10 and 30 respectively.
The BERT revision and asset names are listed in the reproducibility records.
Their stopping decisions use validation data; BERT trains sequentially, whereas SimpleCIL trains its representation only at the first task.
For GCN, BERT, Cosine, TEEN, and SimpleCIL, saving and restoring the validation-selected checkpoint verifies every trainable tensor; empty checkpoints are rejected.
Cross-graph prototype routines requesting fixed sample counts use all prescribed current-class training rows.
The single-graph GCN, Cosine, TEEN, BERT, and SimpleCIL runs retain their recorded upstream per-dataset configuration overrides and prototype-sampling settings; their complete effective configurations accompany the results.
For the native whole6 readout regularizers, EWC and MAS use weight 100, learning rate $5\times10^{-4}$, three epochs, and batch size 16; their retained low-rank state uses ranks 64 (EWC) and 512 (MAS/OGD).
The recorded EWC recipe uses external-reference statistics, while MAS/OGD also retain task-derived statistics; whole4/global10 retain their archived recipe provenance rather than being attributed to this native training entry point.

\begin{table}[t]
\centering
\caption{What is and is not shared across the compared implementations. All displayed methods use the prescribed current training roots and cumulative-candidate test protocol; history access and adaptation remain method-specific.}
\label{tab:comparison-conditions}
\small
\renewcommand{\arraystretch}{1.15}
\begin{tabular}{@{}p{.18\linewidth}p{.34\linewidth}p{.41\linewidth}@{}}
\toprule
Method family & Updated representation & Historical information used in adaptation \\
\midrule
HiRP / four-cell ablation & Shared readout on frozen GraphCLIP features & Bounded historical and external unlabeled features; response constraints. \\
SGD / LwF & Same readout and frozen feature cache & SGD: current CE only. LwF: previous-model responses on current inputs; no historical-node CE. \\
EWC / MAS / OGD & Shared readout & Stored importance/projection information under each saved recipe; not the HiRP response constraint. \\
LoRA / G2LoRA$^{\dagger}$ & Graph adapters/readout; selected graph and text adapters & Published-mechanism adaptations with their own retained state and optimization budgets. \\
InfLoRA$^{\dagger}$ & Last graph-attention K/V adapters and task-wise linear heads & Activation subspaces and frozen prior factors; no historical supervised replay. \\
SD-LoRA & Last graph-attention Q/V adapters and task-wise linear heads & Frozen prior low-rank directions with learned magnitudes; no historical supervised replay. \\
GCN / BERT & GCN and head / text LoRA and head & Sequential current-task training; distinct from the pretrained GraphCLIP readout. \\
GCN\textsubscript{LLMEmb} & GCN and classifier on frozen LLM node embeddings & Sequential current-task training; no historical supervised replay. \\
RoBERTa & Full RoBERTa-large encoder and classification head & Current-task supervised fine-tuning; no historical supervised replay. \\
Cosine / TEEN / SimpleCIL & Method-specific representations and prototype rules & Retained class prototypes; SimpleCIL fixes BERT after first-task adaptation. \\
TPP-style & Independent input-feature linear heads & Retained heads and graph-derived task profiles; predicted per-node routing. \\
GraphGPT$^{\dagger}$ & Linear projector; frozen GraphCLIP and Llama-3.2-3B & Shared source matching initialization; sequential current-task tuning without historical supervised replay. \\
LLaGA$^{\dagger}$ & ND projector; MiniLM384 inputs and frozen Llama-3.2-3B & Per-seed cached neighbor samples; sequential projector tuning without historical supervised replay. \\
\bottomrule
\end{tabular}
\end{table}

HiRP performs 6,219, 13,110, and 6,228 current-data updates over three epochs on whole4, whole6, and global10.
Each update includes Fisher-vector products, up to 12 CG iterations per solve, an optional second solve, and up to 48 scalar-search iterations per active KL constraint.
Inference uses the same frozen encoder and shared readout as its readout baselines, without the training support or solver.
Table~\ref{tab:recorded-runtime} compares recorded training times with GraphGPT, LLaGA, InfLoRA, and SD-LoRA under their documented configurations.
Baseline times sum the recorded training epochs; HiRP uses task timers that also include support updates and checkpoint saving.
HiRP has lower recorded training time in all three streams, with feature precomputation, validation, and testing excluded.
Training budgets, adapted modules, and device loads differ across methods.

\begin{table}[t]
\centering
\caption{Recorded training time per complete sequence in minutes (mean $\pm$ sample standard deviation). Feature precomputation, validation, and testing are excluded.}
\label{tab:recorded-runtime}
\small
\begin{tabular}{lrrr}
\toprule
Method & 4 datasets / 4 tasks & 6 datasets / 6 tasks & 4 datasets / 10 tasks \\
\midrule
GraphGPT$^{\dagger}$ & 250.9 $\pm$ 0.8 & 742.2 $\pm$ 128.2 & 217.5 $\pm$ 26.4 \\
LLaGA$^{\dagger}$ & 174.5 $\pm$ 59.7 & 523.9 $\pm$ 67.5 & 142.2 $\pm$ 19.6 \\
InfLoRA$^{\dagger}$ & 19.4 $\pm$ 2.0 & 31.8 $\pm$ 3.6 & 27.2 $\pm$ 1.7 \\
SD-LoRA & 21.2 $\pm$ 0.3 & 37.2 $\pm$ 5.0 & 34.2 $\pm$ 2.6 \\
\midrule
HiRP (ours) & 6.7 $\pm$ 1.1 & 16.2 $\pm$ 5.5 & 7.2 $\pm$ 2.5 \\
\bottomrule
\end{tabular}
\end{table}

The recorded cached-feature environment is Python 3.10.20, PyTorch 2.2.2 with CUDA 12.1, PyG 2.7.0, Transformers 4.44.2, and V100 32-GB hardware; its version record is supplied with the reproducibility materials.
All four-cell variants reset current-batch RNG to seed $+100000t$ and support selection to seed $+300000t$.
HiRP uses fixed three-epoch final readouts, while heterogeneous baselines retain their documented stopping and adaptation rules.
The supplied configurations and result records identify the evaluated implementations.

\paragraph{Reproducibility materials.}
The accompanying \texttt{reproducibility/} directory contains protocols, split identities, configurations, five-run main-table results, and derived diagnostic data; \texttt{tools/} provides figure reconstruction scripts.
Its README distinguishes outputs reproducible from these records from training runs requiring the original feature caches, pretrained assets, and complete training/evaluation source.
This is a record and analysis supplement, not a self-contained end-to-end training release.

\paragraph{Scope and limitations.}
The three streams test different domain/class schedules but share several datasets.
Appendix~\ref{app:hyperparameter-sensitivity} examines the local and task radii, damping, and support capacity; sensitivity to other orders and backbones remains to be evaluated.
The benchmark informed method development, and the documented baseline recipes need not be optimal for every method.
Historical reverse-order runs add calibration or group supervision and are outside the evaluated HiRP configuration.
At the method level, aggregate current mass permits within-group changes, bounded support can miss historical regions, and task-entry references can already contain candidate-insertion errors.
These limitations separate empirical response preservation from an unconditional guarantee against forgetting.

\subsection{Single-graph classification}
\label{app:single-graph-results}

\begin{table}[t]
\centering
\caption{Single-graph class-incremental classification. All entries use the same class splits and cumulative-candidate evaluation. FA and AA are percentages; bold denotes the largest mean among the displayed methods. The displayed methods retain their respective backbones.}
\label{tab:single-graph}
\begingroup
\scriptsize
\setlength{\tabcolsep}{2.5pt}
\renewcommand{\arraystretch}{1.2}
\resizebox{\linewidth}{!}{%
\begin{tabular}{@{}lcccccccc@{}}
\toprule
Method & \multicolumn{2}{c}{Cora} & \multicolumn{2}{c}{CiteSeer} & \multicolumn{2}{c}{WikiCS} & \multicolumn{2}{c}{Photo} \\
\cmidrule(lr){2-3} \cmidrule(lr){4-5} \cmidrule(lr){6-7} \cmidrule(lr){8-9}
 & FA $\uparrow$ & AA $\uparrow$ & FA $\uparrow$ & AA $\uparrow$ & FA $\uparrow$ & AA $\uparrow$ & FA $\uparrow$ & AA $\uparrow$ \\
\midrule
GCN & 54.61\,\(\pm\)\,0.18 & 31.42\,\(\pm\)\,0.11 & 32.45\,\(\pm\)\,0.31 & 31.36\,\(\pm\)\,0.30 & 43.76\,\(\pm\)\,3.69 & 36.63\,\(\pm\)\,4.14 & 14.58\,\(\pm\)\,0.04 & 16.39\,\(\pm\)\,0.05 \\
Cosine & 65.87\,\(\pm\)\,1.39 & 62.54\,\(\pm\)\,0.98 & 42.79\,\(\pm\)\,4.27 & 38.87\,\(\pm\)\,4.34 & 58.07\,\(\pm\)\,1.38 & 51.58\,\(\pm\)\,1.85 & 56.34\,\(\pm\)\,1.86 & 60.00\,\(\pm\)\,3.01 \\
TEEN & 59.04\,\(\pm\)\,4.36 & 54.27\,\(\pm\)\,3.70 & 54.49\,\(\pm\)\,2.91 & 48.21\,\(\pm\)\,2.67 & 59.76\,\(\pm\)\,1.51 & 55.31\,\(\pm\)\,0.60 & 49.75\,\(\pm\)\,3.19 & 52.89\,\(\pm\)\,0.99 \\
TPP-style & 68.20\,\(\pm\)\,0.59 & 63.53\,\(\pm\)\,0.47 & 70.64\,\(\pm\)\,0.59 & 69.10\,\(\pm\)\,0.55 & 58.10\,\(\pm\)\,1.10 & 58.61\,\(\pm\)\,1.23 & 50.61\,\(\pm\)\,0.04 & 66.57\,\(\pm\)\,0.09 \\
SGD & 67.65\,\(\pm\)\,0.28 & 56.49\,\(\pm\)\,0.70 & 63.85\,\(\pm\)\,0.50 & 60.82\,\(\pm\)\,0.49 & 62.22\,\(\pm\)\,0.06 & 55.13\,\(\pm\)\,0.06 & 28.68\,\(\pm\)\,0.05 & 28.85\,\(\pm\)\,0.07 \\
LwF & 68.88\,\(\pm\)\,0.11 & 59.93\,\(\pm\)\,0.26 & 64.94\,\(\pm\)\,0.48 & 60.72\,\(\pm\)\,0.52 & 62.46\,\(\pm\)\,0.12 & 54.38\,\(\pm\)\,0.12 & 28.65\,\(\pm\)\,0.31 & 31.53\,\(\pm\)\,0.11 \\
LoRA & 69.07\,\(\pm\)\,0.38 & 59.86\,\(\pm\)\,0.37 & 68.86\,\(\pm\)\,0.24 & 65.97\,\(\pm\)\,0.28 & 66.28\,\(\pm\)\,0.25 & 60.07\,\(\pm\)\,0.22 & 34.04\,\(\pm\)\,0.03 & 37.92\,\(\pm\)\,0.05 \\
\midrule
HiRP (ours) & \textbf{\boldmath 75.09\,\(\pm\)\,0.18} & \textbf{\boldmath 72.57\,\(\pm\)\,0.42} & \textbf{\boldmath 73.88\,\(\pm\)\,1.27} & \textbf{\boldmath 70.71\,\(\pm\)\,1.79} & \textbf{\boldmath 74.22\,\(\pm\)\,0.08} & \textbf{\boldmath 72.18\,\(\pm\)\,0.12} & \textbf{\boldmath 71.23\,\(\pm\)\,0.66} & \textbf{\boldmath 75.45\,\(\pm\)\,1.70} \\
\bottomrule
\end{tabular}%
}
\endgroup
\end{table}

\begin{table}[t]
\centering
\caption{Additional language-model comparisons on CiteSeer and Photo. Entries are means and sample standard deviations, using the same evaluation splits as Table~\ref{tab:single-graph}.}
\label{tab:single-language}
\begingroup
\scriptsize
\setlength{\tabcolsep}{2.5pt}
\renewcommand{\arraystretch}{1.2}
\begin{tabular*}{\linewidth}{@{\extracolsep{\fill}}lcccc@{}}
\toprule
Method & \multicolumn{2}{c}{CiteSeer} & \multicolumn{2}{c}{Photo} \\
\cmidrule(lr){2-3} \cmidrule(lr){4-5}
 & FA $\uparrow$ & AA $\uparrow$ & FA $\uparrow$ & AA $\uparrow$ \\
\midrule
BERT & 31.24\,\(\pm\)\,0.36 & 30.20\,\(\pm\)\,0.35 & 14.29\,\(\pm\)\,0.01 & 16.06\,\(\pm\)\,0.01 \\
SimpleCIL (BERT) & 57.58\,\(\pm\)\,4.65 & 57.75\,\(\pm\)\,4.23 & 31.98\,\(\pm\)\,6.98 & 43.63\,\(\pm\)\,5.02 \\
\midrule
HiRP (ours) & \textbf{\boldmath 73.88\,\(\pm\)\,1.27} & \textbf{\boldmath 70.71\,\(\pm\)\,1.79} & \textbf{\boldmath 71.23\,\(\pm\)\,0.66} & \textbf{\boldmath 75.45\,\(\pm\)\,1.70} \\
\bottomrule
\end{tabular*}
\endgroup
\end{table}

HiRP attains the highest FA and AA among the methods in Table~\ref{tab:single-graph} on all four graphs.
Its FA is $75.09\%$, $73.88\%$, $74.22\%$, and $71.23\%$ on Cora, CiteSeer, WikiCS, and Photo, respectively.
The same configuration is used without substituting dataset-specific exploratory variants.
Table~\ref{tab:single-language} additionally reports BERT and SimpleCIL (BERT) on CiteSeer and Photo using the same evaluation splits.
These results extend the classification comparison to individual graphs.

\subsection{Empirical coverage of the prediction-stability certificate}
\label{app:certificate-coverage}

We evaluate the practical coverage of Theorem~\ref{prop:response-stability} on the four single-graph HiRP settings in Table~\ref{tab:single-graph}, covering 12 sequences and 39 task transitions.
For each transition, we compare saved test logits from the task-entry and task-end readouts on identical historical nodes and the same expanded candidate set $\mathcal O_t\cup\mathcal N_t$.
The reference-correct set $\mathcal I_{t,s}$ contains historical test nodes whose task-entry prediction over this expanded set equals their true label $y$.
Correctness restricted to historical candidates is not sufficient for inclusion.
We compute $\delta_R$ in Eq.~\eqref{eq:response-margin} with $c=y$ and the pointwise $D_H$ from the paired logits, retaining the original score scale without temperature adjustment.
Ground-truth labels are used only for this offline diagnostic.

\paragraph{Coverage and aggregation.}
Writing $N_{\mathrm{ref}}=|\mathcal I_{t,s}|$, the two coverage rates are
\begin{equation}
    C_+=\frac{\sum_{x\in\mathcal I_{t,s}}\mathbf 1\{\delta_R(x)>0\}}{N_{\mathrm{ref}}},
    \qquad
    C_{\mathrm{cert}}=\frac{\sum_{x\in\mathcal I_{t,s}}\mathbf 1\{\delta_R(x)>0,\ D_H(x)<\delta_R(x)^2/2\}}{N_{\mathrm{ref}}}.
\end{equation}
Table~\ref{tab:certificate-coverage} first averages each rate equally over transitions within a run, then reports the mean and sample standard deviation across runs.
Its counts instead pool sample--transition pairs; a node may occur in multiple transitions, so these are not counts of independent nodes.
We evaluate the strict inequalities in float64 using stable log-softmax/logsumexp and natural logarithms.
No reference-correct pair lies within $10^{-12}$ of the zero-margin or certificate boundary; this tolerance is used only for auditing, not to relax either condition.

\begin{table}[t]
\centering
\caption{Single-graph certificate coverage. $N_{\mathrm{ref}}$ counts reference-correct historical sample--transition pairs on expanded candidates; Trans. counts transitions across runs. $C_+$ and $C_{\mathrm{cert}}$ are percentages (transition mean within each run, then mean $\pm$ sample SD). The last column reports pooled correct-to-incorrect changes among certified pairs.}
\label{tab:certificate-coverage}
\begingroup
\small
\setlength{\tabcolsep}{3pt}
\renewcommand{\arraystretch}{1.16}
\begin{tabular*}{\linewidth}{@{\extracolsep{\fill}}lr r c c r@{}}
\toprule
Dataset & Trans. & $N_{\mathrm{ref}}$ & $C_+$ (\%) & $C_{\mathrm{cert}}$ (\%) & Forgotten / certified \\
\midrule
Cora & 6 & 724 & 97.63 $\pm$ 0.79 & 88.89 $\pm$ 2.54 & 0 / 642 \\
CiteSeer & 6 & 1,107 & 98.28 $\pm$ 1.80 & 91.82 $\pm$ 1.91 & 0 / 1,032 \\
WikiCS & 12 & 21,565 & 98.21 $\pm$ 0.16 & 96.42 $\pm$ 0.24 & 0 / 20,717 \\
Photo & 15 & 44,367 & 96.41 $\pm$ 0.18 & 80.86 $\pm$ 0.48 & 0 / 35,906 \\
\bottomrule
\end{tabular*}
\endgroup
\end{table}

\paragraph{Observed retention.}
The complete certificate covers $80.86\%$--$96.42\%$ under the stated averaging rule.
Across all datasets, 58,297 of 67,763 reference-correct sample--transition pairs satisfy both conditions, and none changes from correct to incorrect.
Another 5,856 pairs remain correct without certification; all 3,610 observed correct-to-incorrect changes occur among uncertified pairs.
Figure~\ref{fig:certificate-coverage} plots all 66,396 reference-correct pairs with positive response margins; the remaining 1,367 pairs are excluded from the plot but remain in the coverage denominators.
Every plotted point strictly below the diagonal retains its correct prediction, consistent with Theorem~\ref{prop:response-stability}; retained points above the diagonal illustrate that the condition is sufficient rather than necessary.

\begin{figure}[t]
    \centering
    \includegraphics[width=\linewidth]{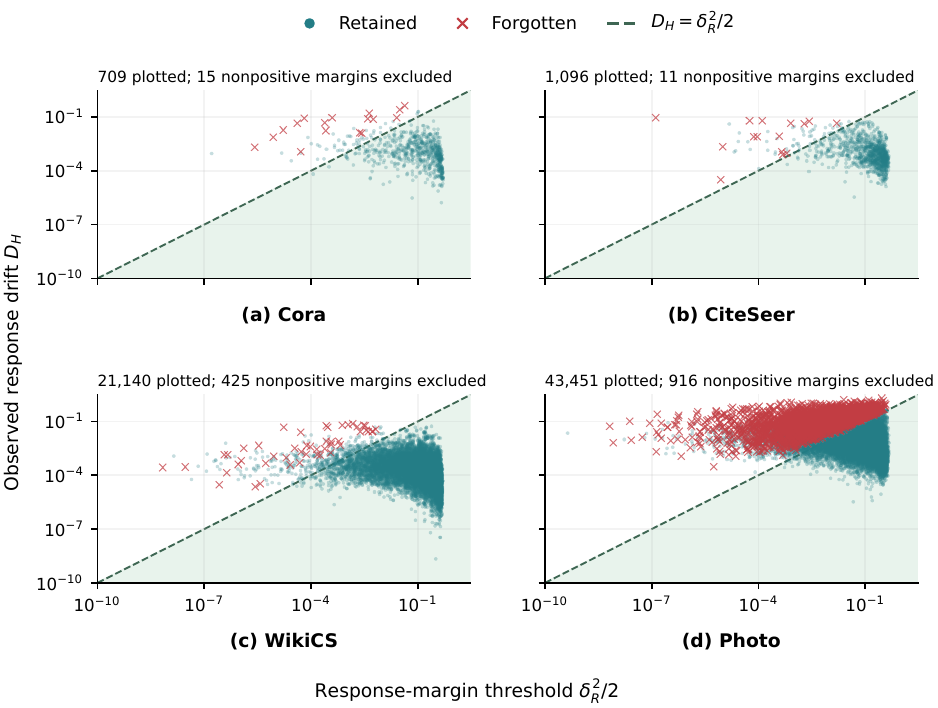}
    \caption{Single-graph coverage of Theorem~\ref{prop:response-stability}. Each point is a reference-correct historical sample--transition pair with $\delta_R>0$. Circles indicate retained predictions and crosses indicate correct-to-incorrect changes. The dashed line is $D_H=\delta_R^2/2$; the shaded region below it satisfies the strict certificate and contains no forgotten pairs. Panel annotations give the number of nonpositive-margin pairs excluded from the plot. Both axes use logarithmic scales.}
    \label{fig:certificate-coverage}
\end{figure}

\paragraph{Scope.}
Coverage is measured on the four single-graph settings using observed test-point margins and task-entry-to-task-end drift after training.
It assesses a sufficient condition for individual prediction stability, distinct from full-sequence signed AF and causal attribution of algorithmic gains.
The analysis is limited to these settings; an average training-support budget alone does not certify test points.
Per-transition counts, derived point data, and plotting code accompany the reproducibility supplement.

\subsection{Final accuracy and baseline retention}
\label{app:baseline-retention}

Table~\ref{tab:baseline-retention} separates final historical-task accuracy $R$ from acquisition accuracy $Q$, with $\mathrm{AF}=R-Q$.
The smaller forgetting of TPP-style does not arise solely from a lower acquisition reference: on whole4 it also has higher final historical accuracy than HiRP ($68.28\%$ versus $66.48\%$).
HiRP instead achieves higher last-task accuracy ($85.05\%$ versus $72.22\%$), yielding higher overall FA and AA.
On whole6 and global10, HiRP has both higher acquisition and higher final historical accuracy, alongside a larger decline from acquisition.
These results characterize a learning--retention trade-off while retaining final classification performance as the primary objective.

\begin{table}[t]
\centering
\caption{Historical-task retention for baseline variants and the hierarchical-response ablation. $R$ is final historical-task accuracy and $Q$ is accuracy when each historical task was first learned; the final task is excluded. Entries report means $\pm$ sample standard deviations in percentages or percentage points, and $\mathrm{AF}=R-Q$ before rounding. Bold marks the largest displayed mean.}
\label{tab:baseline-retention}
\label{tab:retention-decomposition}
\begingroup
\fontsize{8}{9.5}\selectfont
\setlength{\tabcolsep}{0.6pt}
\renewcommand{\arraystretch}{1.2}
\resizebox{\linewidth}{!}{%
\begin{tabular}{@{}lccccccccc@{}}
\toprule
Method & \multicolumn{3}{c}{4 datasets / 4 tasks} & \multicolumn{3}{c}{6 datasets / 6 tasks} & \multicolumn{3}{c}{4 datasets / 10 tasks} \\
\cmidrule(lr){2-4} \cmidrule(lr){5-7} \cmidrule(lr){8-10}
 & $R$ $\uparrow$ & $Q$ $\uparrow$ & AF $\uparrow$ & $R$ $\uparrow$ & $Q$ $\uparrow$ & AF $\uparrow$ & $R$ $\uparrow$ & $Q$ $\uparrow$ & AF $\uparrow$ \\
\midrule
TPP-style & \textbf{\boldmath 68.28\(\pm\)0.34} & 70.57\(\pm\)0.34 & \(-\)2.28\(\pm\)0.04 & 70.79\(\pm\)0.06 & 72.19\(\pm\)0.05 & \(-\)1.40\(\pm\)0.02 & 59.60\(\pm\)0.19 & 69.58\(\pm\)0.24 & \(-\)9.98\(\pm\)0.09 \\
SimpleCIL & 59.83\(\pm\)1.15 & 61.59\(\pm\)1.19 & \textbf{\boldmath \(-\)1.76\(\pm\)0.05} & 48.57\(\pm\)0.73 & 49.90\(\pm\)0.79 & \textbf{\boldmath \(-\)1.33\(\pm\)0.07} & 50.75\(\pm\)0.95 & 58.51\(\pm\)0.87 & \textbf{\boldmath \(-\)7.77\(\pm\)0.09} \\
Hier. + radial & 66.38\(\pm\)0.69 & 74.17\(\pm\)0.63 & \(-\)7.79\(\pm\)0.41 & 71.53\(\pm\)0.33 & 76.70\(\pm\)0.26 & \(-\)5.17\(\pm\)0.17 & \textbf{\boldmath 67.65\(\pm\)0.16} & 80.42\(\pm\)0.12 & \(-\)12.76\(\pm\)0.27 \\
\midrule
HiRP (ours) & 66.48\(\pm\)0.37 & \textbf{\boldmath 74.91\(\pm\)0.57} & \(-\)8.43\(\pm\)0.32 & \textbf{\boldmath 73.45\(\pm\)0.13} & \textbf{\boldmath 79.66\(\pm\)0.25} & \(-\)6.21\(\pm\)0.32 & 67.26\(\pm\)0.09 & \textbf{\boldmath 82.00\(\pm\)0.11} & \(-\)14.74\(\pm\)0.15 \\
\bottomrule
\end{tabular}%
}
\endgroup
\end{table}

Saved per-node predictions show that all historical correct-to-incorrect changes in the TPP-style runs coincide with changed routing; predictions remain unchanged when the selected head is unchanged.
Its mean final task-routing accuracies are $98.94\%$, $96.46\%$, and $63.54\%$ on whole4, whole6, and global10.
These diagnostics characterize the routing-based TPP-style variant defined in Appendix~\ref{app:baseline-settings}.

\subsection{Retention and acquisition in the ablation}
\label{app:retention-decomposition}

Let $R=(T-1)^{-1}\sum_{i<T}A_{T,i}$ be final historical-task accuracy and $Q=(T-1)^{-1}\sum_{i<T}A_{i,i}$ the corresponding acquisition accuracy.
For paired configurations on the same tasks, $\Delta\mathrm{AF}=\Delta R-\Delta Q$.
Table~\ref{tab:retention-decomposition} reports both configurations; tangent-minus-radial differences are paired comparisons with the hierarchical response held fixed.
Whole4 and whole6 improve final historical accuracy but increase acquisition accuracy by more, yielding lower AF.
Global10 additionally incurs a small final-retention loss.

The mean KL budget controls the empirical support average, so changes concentrated on poorly covered historical regions can remain consequential for accuracy.

\clearpage
\subsection{Functional preservation with identical feature support}
\label{app:memory-distillation}

Table~\ref{tab:memory-distillation} compares HiRP with two distillation controls using the same historical and external feature support, alongside SGD and LwF.

\begin{table}[ht]
\centering
\caption{Additional distillation controls with identical feature support. Means $\pm$ sample standard deviations; FA/AA are accuracy (\%) and AF is signed forgetting (points). Bold marks the highest mean.}
\label{tab:memory-distillation}
\begingroup
\fontsize{8}{9.5}\selectfont
\setlength{\tabcolsep}{0.6pt}
\renewcommand{\arraystretch}{1.2}
\resizebox{\linewidth}{!}{%
\begin{tabular}{@{}lccccccccc@{}}
\toprule
Method & \multicolumn{3}{c}{4 datasets / 4 tasks} & \multicolumn{3}{c}{6 datasets / 6 tasks} & \multicolumn{3}{c}{4 datasets / 10 tasks} \\
\cmidrule(lr){2-4} \cmidrule(lr){5-7} \cmidrule(lr){8-10}
 & FA $\uparrow$ & AA $\uparrow$ & AF $\uparrow$ & FA $\uparrow$ & AA $\uparrow$ & AF $\uparrow$ & FA $\uparrow$ & AA $\uparrow$ & AF $\uparrow$ \\
\midrule
SGD & 75.79\(\pm\)0.18 & 58.35\(\pm\)0.22 & \(-\)25.67\(\pm\)0.08 & 83.02\(\pm\)0.14 & 58.77\(\pm\)0.28 & \(-\)22.70\(\pm\)0.10 & 50.16\(\pm\)0.15 & 49.96\(\pm\)0.15 & \(-\)41.67\(\pm\)0.18 \\
LwF & 76.63\(\pm\)0.18 & 58.43\(\pm\)0.61 & \(-\)25.24\(\pm\)0.69 & 84.14\(\pm\)0.07 & 61.55\(\pm\)0.26 & \(-\)19.06\(\pm\)0.18 & 56.31\(\pm\)0.33 & 49.64\(\pm\)0.30 & \(-\)41.23\(\pm\)0.27 \\
\midrule
Memory-LwF & 77.13\(\pm\)0.17 & 59.37\(\pm\)0.63 & \(-\)24.14\(\pm\)0.65 & 84.99\(\pm\)0.08 & 63.63\(\pm\)0.48 & \(-\)16.75\(\pm\)0.36 & 54.47\(\pm\)0.11 & 47.39\(\pm\)0.39 & \(-\)44.13\(\pm\)0.38 \\
Memory-KD & 78.18\(\pm\)0.13 & 68.14\(\pm\)0.22 & \(-\)11.19\(\pm\)0.27 & 85.46\(\pm\)0.12 & 70.58\(\pm\)0.17 & \(-\)7.63\(\pm\)0.19 & 60.71\(\pm\)0.12 & 58.65\(\pm\)0.42 & \(-\)29.97\(\pm\)0.45 \\
\midrule
HiRP (ours) & \textbf{\boldmath 80.09\(\pm\)0.19} & \textbf{\boldmath 71.12\(\pm\)0.26} & \textbf{\boldmath \(-\)8.43\(\pm\)0.32} & \textbf{\boldmath 87.04\(\pm\)0.42} & \textbf{\boldmath 73.22\(\pm\)0.23} & \textbf{\boldmath \(-\)6.21\(\pm\)0.32} & \textbf{\boldmath 75.43\(\pm\)0.05} & \textbf{\boldmath 69.86\(\pm\)0.13} & \textbf{\boldmath \(-\)14.74\(\pm\)0.15} \\
\bottomrule
\end{tabular}%
}
\endgroup
\end{table}

\paragraph{Setup.}
The teacher is each control's own task-entry readout $\boldsymbol{\theta}_0$, fixed within the task.
Memory-LwF distills softmax probabilities over historical classes $\mathcal O_t$; Memory-KD uses all arrived classes $\mathcal C_{\leq t}$.
Writing the temperature-$\tau$ distribution as $\boldsymbol{\pi}_{t,\tau}$, both minimize
\begin{equation}
    \mathcal L_{\mathrm{CE}}(\mathcal B_t;\boldsymbol{\theta})
    +\beta\tau^2\widehat{\mathbb E}_{x\sim\mathcal M_t}
    D_{\mathrm{KL}}\!\left(\boldsymbol{\pi}_{t,\tau}(x;\boldsymbol{\theta}_0)\,\|\,\boldsymbol{\pi}_{t,\tau}(x;\boldsymbol{\theta})\right),
    \label{eq:memory-distillation}
\end{equation}
where $\widehat{\mathbb E}$ averages all retained support rows with the source weights in Equation~\eqref{eq:empirical-functional-divergence}.
The distillation term starts at task two.
Both controls share HiRP's task-wise support (at most 1024 unlabeled rows), pretrained initialization, frozen features and class semantics, current training splits and batch order, three epochs, batch size 16, and label smoothing $0.2$.
They use ordinary SGD at $5\times10^{-4}$ with zero momentum and weight decay, temperature $\tau=2$, and weight $\beta=1$.
Final task readouts are evaluated with the same cumulative candidates and task-agnostic inference.
The controls run on a T4 with PyTorch 2.5; detailed configurations accompany the supplemental records.

\paragraph{Results.}
Memory-KD improves over Memory-LwF and current-input LwF on all three streams.
HiRP achieves higher FA, AA, and signed AF than both controls; its AA gains over Memory-KD are $2.98$, $2.65$, and $11.22$ points on whole4, whole6, and global10, respectively.
These additional comparisons support HiRP's effectiveness with historical feature access held fixed.

\clearpage
\subsection{Hyperparameter sensitivity analysis}
\label{app:hyperparameter-sensitivity}

We vary the local radius $\delta\in\{0.001,0.002,0.004,0.008\}$, task radius $\rho\in\{0.05,0.1,0.2,0.4\}$, damping $\lambda\in\{0.003,0.01,0.03,0.1\}$, and support capacity $M\in\{768,1024,1536,2048\}$ on whole4, whole6, and global10.
Each run changes one parameter while keeping the others at their defaults: $\delta=0.002$, $\rho=0.1$, $\lambda=0.01$, and $M=1024$.
The study uses complete streams and the same three-epoch training protocol; the external-reference cap remains 512 rows throughout the capacity sweep.

\begin{figure}[htbp]
    \centering
    \includegraphics[width=\linewidth]{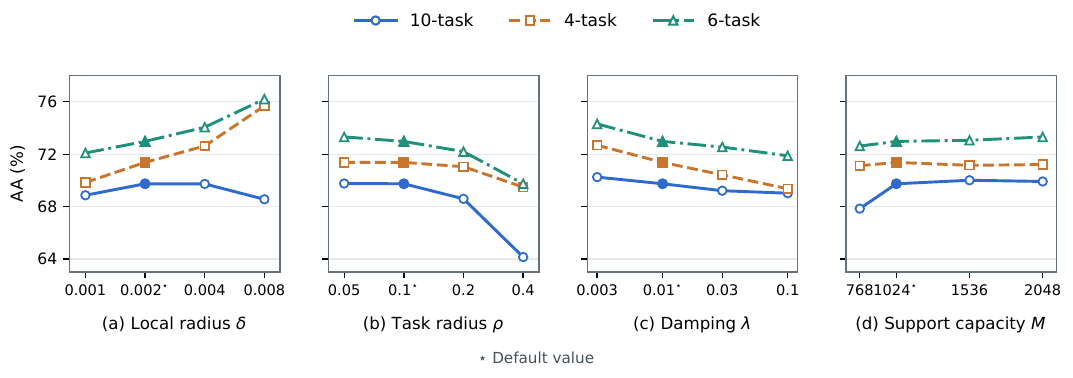}
    \caption{Sensitivity to four core parameters on three continual-learning streams. Curves show final AA; stars mark default values. The first three horizontal axes use logarithmic scales, and support capacity uses a linear scale.}
    \label{fig:core-parameter-sensitivity}
\end{figure}

Figure~\ref{fig:core-parameter-sensitivity} shows that larger local steps improve AA on whole4 and whole6, whereas global10 peaks near the default local radius and declines at $\delta=0.008$.
Increasing the task radius from $0.05$ to $0.4$ reduces AA on all three streams, most strongly on global10, from $69.76\%$ to $64.15\%$.
Lower damping gives higher AA over the tested range.
Increasing capacity from 768 to 1024 improves global10 AA from $67.85\%$ to $69.74\%$, with limited further gains; whole4 changes little and whole6 improves modestly.
These trends show that local step size, cumulative drift, and support coverage affect the streams differently.

\end{document}